\documentclass[conference]{IEEEtran}

\usepackage{cite}
\usepackage{amsmath,amssymb,amsfonts}
\usepackage{graphicx}
\usepackage{textcomp}
\usepackage{xcolor}
\def\BibTeX{{\rm B\kern-.05em{\sc i\kern-.025em b}\kern-.08em
		T\kern-.1667em\lower.7ex\hbox{E}\kern-.125emX}}

\usepackage{booktabs} % For formal tables

\usepackage[utf8]{inputenc}

\usepackage{algorithm}
\usepackage{balance}
\usepackage[noend]{algpseudocode}

\makeatletter
\algrenewcommand\ALG@beginalgorithmic{\small}
\makeatother

\usepackage{pifont}
\newcommand{\cmark}{\ding{51}}
\newcommand{\xmark}{\ding{55}}
\usepackage{amsthm}
\usepackage{mathtools}
\usepackage{hyperref}

\usepackage{stmaryrd}

\usepackage{subcaption}

\newtheorem{lemma}{Lemma}

\newtheorem{theorem}{Theorem}

\begin{document}

\title{Monte Carlo Dependency Estimation}

\author{\IEEEauthorblockN{Edouard Fouché \& Klemens Böhm}
	\IEEEauthorblockA{Institute for Program Structures and Data Organization \\
		Karlsruhe Institute of Technology (KIT), Germany \\
		\{edouard.fouche, klemens.boehm\}@kit.edu}}

\maketitle

%Add page numbers (uncomment for the actual submission)
\thispagestyle{plain}
\pagestyle{plain}

\newtheorem{dfn}{Definition}
\newtheorem{exmp}{Example}

\begin{abstract}
	
Estimating the dependency of variables is a fundamental task in data analysis. Identifying the relevant attributes in databases leads to better data understanding and also improves the performance of learning algorithms, both in terms of runtime and quality. In data streams, dependency monitoring provides key insights into the underlying process, but is challenging. 
In this paper, we propose \textit{Monte Carlo Dependency Estimation} (\textit{MCDE}), a theoretical framework to estimate multivariate dependency in static and dynamic data. \textit{MCDE} quantifies dependency as the average discrepancy between marginal and conditional distributions via Monte Carlo simulations. Based on this framework, we present \textit{Mann-Whitney\,P} (\textit{MWP}), a novel dependency estimator. We show that \textit{MWP} satisfies a number of desirable properties and can accommodate any kind of numerical data. We demonstrate the superiority of our estimator by comparing it to the state-of-the-art multivariate dependency measures.

\end{abstract}

%\begin{IEEEkeywords}
%	Multivariate statistics, Exploratory data analysis, Data mining, Unsupervised learning
%\end{IEEEkeywords}

\section{Introduction}

\subsection{Motivation}
\label{motivation}
Estimating statistical relationships between variables is fundamental to any knowledge discovery process and has become an important topic in the database community \cite{chen1996data, zhu2002statstream, hall2003benchmarking}. Knowing the relationship between attributes, one can infer useful knowledge about unknown outcomes. For example, knowing that weight and arterial pressure correlate
with the odds of contracting certain diseases may guide physicians, to predict whether a patient will become sick within a year or not. 

Modern database systems gather and store data at very high rates. With predictive maintenance for instance, data often is a stream produced in real-time by multiple sensors. 
In this setting, the timely detection of changes in the stream is crucial. The early discovery of anomalies can lead to, say, faster recovery and tremendous cost savings. Real-time detection is challenging because of a phenomenon known as \textit{concept drift}~\cite{Barddal2015}: The data distribution and correlation structure can change over time, be it gradually, be it abruptly, in unexpected ways. A real-world example illustrates this: 

\begin{figure}[]
	\centering
	\includegraphics[width=0.98\linewidth]{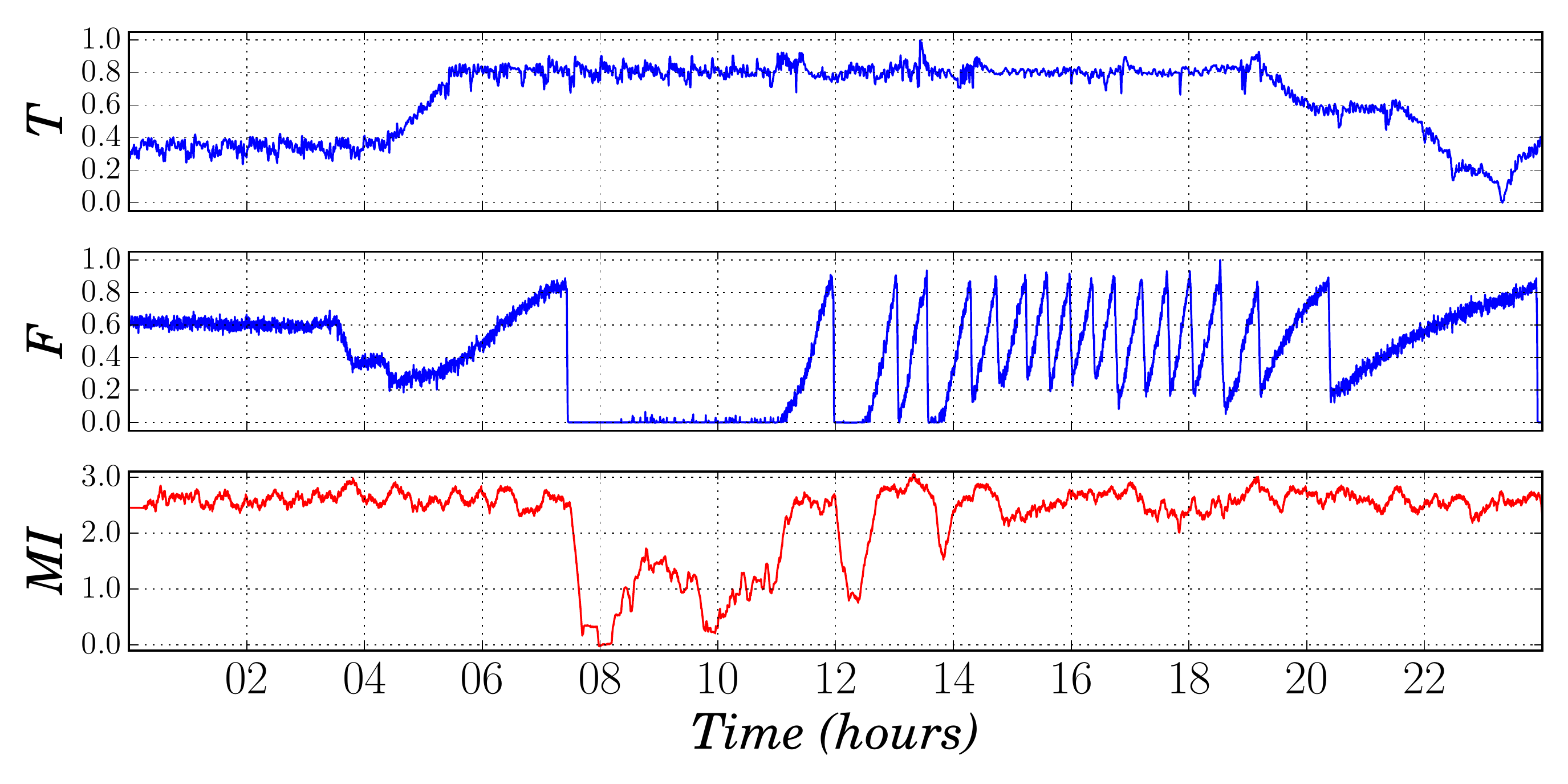}
	\caption{Example of concept drift in dependency monitoring.}
	\label{fig:bioliq}
\end{figure} 

\begin{exmp}
	\label{exp1}
	Let \textit{T} and \textit{F} be two sensor streams in a pyrolytic plant. Stream \textit{T} is the temperature in a reactor, while stream \textit{F} is the filling level of the flue gas cyclone connected to its output. Figure~\ref{fig:bioliq} graphs a 24-hours time span with a sampling rate of one second, i.e., $24 \cdot 60^2 = 86400$ values in total, which we scale to $[0,1]$. We report the Mutual Information (\textit{MI}) \cite{shannon1951mathematical} between \textit{T} and \textit{F} at any time over the last 15 minutes, using a sliding window over the last $900$ values. \textit{MI} quantifies the information shared by both variables. 
    At the beginning, the reactor heats up to its operational temperature. The material introduced into the reactor leads to the production of flue gas, stored temporarily in the cyclone for further processing. The \textit{MI} of the two streams suddenly drops from 2.5 bits to 0 at 7:45. 
    The cyclone does not seem to operate as it should, i.e., as in the later time span between 12:00 and 20:00. 
    This is a sign of interruption in the production process. Such interruptions can become very costly if unnoticed. Thus, a careful monitoring of the plant elements is essential, as drifting dependencies might indicate abnormal events. 
\end{exmp}

While the example only features two streams of data, effects over multiple streams are interesting as well. In the real world, data streams often are an open-ended, ever evolving collection of sensor signals. 
The signals can be noisy, redundant or generated at a varying speed. 
In such contexts, one needs a dependency estimator satisfying all requirements described next. To our knowledge, any existing solution only fulfils some of them at best. In this article, we propose a new estimator with all these characteristics.

\textbf{R1: Multivariate.} Bivariate dependency measures \cite{Szekely2009, Reshef2011} only apply to attribute pairs. Estimating the dependency between more than two attributes is useful 
as well, but existing attempts to generalize bivariate measures lack efficiency or effectiveness, as we will show in this paper.
	
\textbf{R2: Efficient.} For monitoring, one needs to estimate dependency `at least as fast' as the stream. Next, one often is not only interested in a particular set of attributes, but potentially all of them. Since the number of attribute combinations grows exponentially with the number of attributes, the efficiency of the estimator is crucial, with large data streams in particular.
	
\textbf{R3: General-purpose.} Dependency estimators should not be restricted to specific types of dependency. Existing multivariate estimators are typically limited, e.g., \cite{Schmid2007} can only detect monotonous dependencies and \cite{Nguyen2015UDS} only functional ones. 
	
\textbf{R4: Intuitive.} A method is intuitive if its parameters are easy to set, i.e., users understand their impact on the estimation process. Existing solutions typically require a number of unintuitive parameters, and the suggestion of `good' pa\-ra\-me\-ter values often happens at the discretion of the inventors. Different values often yield very different results. Hence, we target at a method that is intuitive to use. 
	
\textbf{R5: Non-parametric.} Since real data can exhibit virtually any kind of distribution, it is not reasonable to use measures relying on parametric assumptions. The risk is to systematically miss relevant effects with wrong assumptions. 
	
\textbf{R6: Interpretable.} The results of dependency estimators should be interpretable: There should be a maximum and a minimum, such that one can easily interpret a given estimate from `highly dependent' to `independent'. 
	
\textbf{R7: Sensitive.} Dependency estimation is not only about deciding whether a relationship exists, but also about quantifying its strength. Database entries generally are observations sampled from a potentially noisy process.
The same dependency should get a higher score when observed with more objects, as the size of the observed effect is larger. 
	
\textbf{R8: Robust.} Real-world data may be of poor quality. It is common to discretise attributes, for a more compact representation. 
Next, measuring devices often have a limited precision, such that values are rounded or trimmed, wrongly leading to data points with exactly the same values. 
Such artefacts can have a negative influence on the estimation. Estimators need to be robust against duplicates and imprecision. 
	
\textbf{R9: Anytime flexibility.} A database may be too large to allow for acceptable computation times, and the rate of incoming items from a data stream may vary. Users should be able to trade accuracy for a faster computation and to interrupt the estimation process at any time. 
In other words, users can set a `budget' they are willing to spend. Conversely, the estimator should return approximate results, ideally with a known quality, in case of early termination. 

\subsection{Contributions}

%The contributions of this article are as follows: 

\textbf{We introduce a framework to estimate multivariate dependency, named \textit{Monte Carlo Dependency Estimation} \textit{(MCDE)}}.
\textit{MCDE} quantifies the dependency of an attribute set as the average discrepancy between marginal and conditional distributions via Monte Carlo simulations. Iteratively, a condition is applied on each dimension, in a process called \textit{subspace slicing} \cite{Keller2012}. A statistical test quantifies the discrepancy between the marginal and conditional distributions of a dimension taken at random. \textit{MCDE} is abstract, since the underlying statistical test is left unspecified. We determine a lower bound for the quality of the estimation, allowing to trade a quantifiable level of accuracy for a computational advantage. 

\textbf{As a proof of concept, we instantiate a new dependency measure within \textit{MCDE}, named \textit{Mann-Whitney\,P} (\textit{MWP})}. \textit{MWP} relies on the \textit{Mann-Whitney U} test, a well-known non-parametric statistical test \cite{Mann1947}, to quantify the average discrepancy between the marginal and  conditional distributions. We describe the implementation of \textit{MWP} in detail.

\textbf{We compare our estimator to the state of the art}. We benchmark each approach using an assortment of synthetic dependencies.
In particular, we measure the statistical power and execution time of each approach. This will show that \textit{MWP} fulfils all requirements while the existing ones do not.

\textbf{We release our source code and experiments on GitHub\footnote{\url{https://github.com/edouardfouche/mcde}}} with documentation to ensure reproducibility.

\smallskip 
Paper outline: Section \ref{relatedwork} reviews related work. Section~\ref{mcde} describes \textit{MCDE} and \textit{MWP}. Section \ref{evaluation} evaluates \textit{MWP} and compares it to the state of the art. Section \ref{conclusions} concludes. 

\section{Related Work}
\label{relatedwork}

Estimating the correlation of a set of attributes has been of interest for more than a century. Many bivariate measures exist \cite{Spearman1904, Szekely2009, Reshef2011, Lopez-Paz2013}; 
a famous one is Pearson's \textit{r}, also commonly known as \textit{Pearson correlation coefficient}. However, they are not applicable to multivariate analysis (\textbf{R1}). Also, they often have other drawbacks. For example, Pearson's \textit{r} is parametric  (\textbf{R5}) and targets at linear dependencies (\textbf{R3}).

There exist attempts to extend bivariate dependency measures to the multivariate case. Schmid \textit{et al.}\
\cite{Schmid2007} propose an extension of Spearman's $\rho$ to multivariate data (\textit{MS}), but this is still limited to monotonous relationships (\textbf{R3}). Several authors also propose extensions of \textit{MI} \cite{Timme2014}. For example, Interaction Information (\textit{II}) \cite{McGill1954} quantifies the `synergy' or `redundancy'
in a set of variables. Similarly, Total Correlation (\textit{TC}) \cite{5392532} quantifies the total amount of information. However, information-theoretic measures are difficult to estimate, as they require the knowledge of the underlying probability distributions. Density estimation methods, based on kernels, histograms or local densities, all require to set unintuitive  parameters (\textbf{R4}) and may be computationally expensive (\textbf{R2}). Next, with many dimensions, multivariate density estimation becomes meaningless, due to the \textit{curse of dimensionality} \cite{bellman1957}. 
Information-theoretic measures also are difficult to interpret (\textbf{R6}), since they are unbounded and usually expressed in bits.  

More recently, \textit{CMI} \cite{Nguyen2013CMI}, \textit{MAC} \cite{Nguyen2014MAC} and \textit{UDS} \cite{Nguyen2015UDS} have been proposed as multivariate dependency measures. They are remotely related to concepts from information theory, as they rely on the so-called \textit{cumulative entropy} \cite{Crescenzo2009}. However, these measures are computationally expensive (\textbf{R2}) and not intuitive to use (\textbf{R4}). 
They also are difficult to interpret, because their theoretical maximum and minimum vary with the number of dimensions (\textbf{R6}). To our knowledge, Requirements \textbf{R7-9} have not been considered in the literature so far. 

Another approach, named \textit{HiCS} \cite{Keller2012}, is the one most similar to \textit{MCDE}-\textit{MWP}. It introduces \textit{subspace slicing} as a heuristic to quantify the discrepancy between marginal and conditional distributions. The resulting estimate is used to discover outliers in large databases. Yet the suitability of \textit{HiCS} as a dependency estimators is so far unknown. 

In Section \ref{evaluation}, we compare our estimator \textit{MWP} to the related work, namely \textit{MS}, \textit{TC}, \textit{II}, \textit{CMI}, \textit{MAC}, \textit{UDS} and \textit{HiCS}.

\section{The MCDE Framework} 
\label{mcde}

Dependency estimation determines to which extent a variable relationship differs from randomness. In this spirit, \textit{MCDE} quantifies a dependency, i.e., an extent of independence violation, based on marginal and conditional distributions. 

\subsection{Notation}

Let $DB$ be a database with $n$ objects and $d$ dimensions. It is a set of attributes, or variables, $D=\{s_1, \dots, s_d\}$ and a list of objects $B = (\vec{x}_{1}, \dots, \vec{x}_{n})$ where $\vec{x}_{i} = \langle x_{i}^{s_{j}} \rangle_{j \in \{1, \dots, d\}}$ is a \mbox{$d$-dimensional} vector of real numbers. We call a subspace $S$ a projection of the database on $d'$ attributes, with $S \subseteq D$ and $d' \leq d$. We refer to its dimensionality as $|S| = d'$. 
To formalize our dependency estimator, we treat the attributes in $D$ as random variables, i.e., a random variable $X_{s_j}$ represents each attribute $s_j \in D$.
Additionally, $p(X)$ is the joint probability density function (\textit{pdf}) of a random vector $X = \left \langle X_{s_{i}}\right \rangle  _{s_i \in S}$, and $\hat{p}(X)$ is its estimation. 
We use $p_{s_{j}}(X)$ and $\hat{p}_{s_{j}}(X)$ for the marginal \textit{pdf} of $s_{j}$. $\mathcal{P}(S)$ denotes the power set of $S$, i.e., the set of all attribute subsets. 
For any attribute subset $S' \in \mathcal{P}(S)$, its random vector is $X_{S'}= \left \langle X_{s_i}\right \rangle _{s_{i} \in S'}$, and its complement random vector is $\overline{X_{S'}} = X_{S  \setminus S'} = \left \langle X_{s_i}\right \rangle _{s_{i} \in S \setminus S'}$. 

\subsection{Theory of MCDE} 

\subsubsection{Measuring Dependencies as contrast} 

A set of variables is \textit{independent} or \textit{uncorrelated} if and only if all the variables are \textbf{mutually independent}. By treating the attributes of a subspace as random variables, we can define the independence assumption of a subspace as follows: 
\begin{dfn}[Independence Assumption]
	\label{IA1}
	The independence assumption $\mathcal{H}$ of a subspace $S$  holds if and only if the random variables $\{ X_{s_i} : s_i \in S \}$ are mutually independent, i.e.:   
	\begin{align}
	\mathcal{H}(S) \Leftrightarrow p(X) = \prod_{s_{i} \in S} p_{s_{i}}(X) \label{eq:IA1}
	\end{align}
\end{dfn}
Under the independence assumption, the joint distribution of the subspace $S$ is \textbf{expected} to be equal to the product of its marginal distributions. We can define a degree of dependency, or correlation, based on the degree to which $\mathcal{H}$ does not hold: 
\begin{dfn}[Degree of Dependency]
	\label{DependencyDegree}
	The degree of dependency $\mathcal{D}$ of a subspace $S$ is the discrepancy, abbreviated as $disc$, between the \textbf{observed} joint distribution $p^{o}(X)$ and $p^{e}(X) = \prod_{s_{i} \in S} p^{o}_{s_{i}}(X)$, the \textbf{expected} joint distribution: 
	\begin{align}
	\mathcal{D}(S) &\equiv disc \left(p^{o}(X),p^{e}(X)\right) 
	\end{align}
\end{dfn} 

While one can estimate the discrepancy between two probability distributions, using for instance the Kullback-Leibler divergence \cite{kullback1951information}, this is not trivial here because $p^o(X)$ and $p^{e}(X)$ are a priori unknown. We work around this as follows: 

\begin{lemma}
	\label{lemma1}
	The independence assumption $\mathcal{H}$ of a subspace $S$ holds if and only if the joint \textit{pdf} for all  $S' \in \mathcal{P}(S)$  is equal to its joint conditional \textit{pdf} given all other variables $S \setminus S'$:
	\begin{align}
	\mathcal{H}(S) \Leftrightarrow p(X_{S'}|\overline{X_{S'}}) =  p(X_{S'}) && \forall S' \in \mathcal{P}(S)
	\end{align}
\end{lemma}

\begin{proof}
	Since all variables in $S$ are mutually independent, for any $S' \in \mathcal{P}(S)$ we also have $p(X_{S'}) = \prod_{s_{i} \in S'} p_{s_{i}}(X)$, then
	\begin{align*}
	\mathcal{H}(S) &\Leftrightarrow~ p(X) = \prod_{s_{i} \in S} p_{s_{i}}(X) &&~ \\
	\mathcal{H}(S) &\Leftrightarrow~ p(X) = p(X_{S'}) * \prod_{s_{i} \in S \setminus S'} p_{s_{i}}(X) && \forall S' \in \mathcal{P}(S) \\
	\mathcal{H}(S) &\Leftrightarrow~ \frac{p(X)}{p(\overline{X_{S'}})} = p(X_{S'}) && \forall S' \in \mathcal{P}(S)
	\end{align*}
	By the definition of the conditional \textit{pdf}:
	\begin{align*}
	\mathcal{H}(S) &\Leftrightarrow~ p(X_{S'}|\overline{X_{S'}}) =  p(X_{S'}) && \forall S' \in \mathcal{P}(S) \label{lastofproof1} 
	\qedhere
	\end{align*} 
\end{proof}

Lemma \ref{lemma1} provides an alternative definition of $\mathcal{H}$. However, it is still problematic for the following reasons: 
First, one requires multivariate density estimation to estimate $p(X_{S'})$ and $p(X_{S'}|\overline{X_{S'}})$ with $|S'| \geq 1$ in the multivariate case.
Second, even if one could estimate $p(X_{S'})$ and $p(X_{S'}|\overline{X_{S'}})$, estimating densities for all $ S' \in \mathcal{P}(S)$ is intractable.
So we instead relax the problem by considering only subspaces with  $|S'| = 1$, i.e., we only look at the marginal \textit{pdf} of single variables. 

\begin{dfn}[Relaxed Independence Assumption]
	The relaxed independence assumption $\mathcal{H}^*$ of a subspace $S$ holds if and only if the marginal distribution $p_{s_{i}}(X)$ of each variable $s_i \in S$ equals $p_{s_{i}}(X|\overline{X_{s_i}})$, i.e., the conditional \textit{pdf} of $s_i$:
	\begin{align*}
	\mathcal{H}^*(S) \Leftrightarrow p_{s_{i}}(X|\overline{X_{s_i}}) = p_{s_{i}}(X) && \forall s_i \in S 
	\end{align*}
\end{dfn}

\begin{theorem}[Independence Assumption Relaxation] 
	\label{IA2}
	We can relax $\mathcal{H}$ into $\mathcal{H}^*$ for any $S$, such that $\mathcal{H}(S) \Rightarrow \mathcal{H}^*(S)$.
\end{theorem}

\begin{proof}
	Using Lemma \ref{lemma1}:
	\begin{align*}
	\mathcal{H}(S) &\Leftrightarrow p(X_{S'}|\overline{X_{S'}}) =  p(X_{S'}) && \forall S' \in \mathcal{P}(S) &~\\
	\mathcal{H}(S) &\Rightarrow p(X_{S^1}|\overline{X_{S^1}}) = p(X_{S^1}) && \forall S^1 \in \mathcal{P}(S) :|S^1| = 1 &~\\
	\mathcal{H}(S) &\Rightarrow p_{s_i}(X|\overline{X_{s_i}}) = p_{s_{i}}(X) && \forall s_i \in S 
	&\qedhere
	\end{align*}
\end{proof}

Loosely speaking, the relaxed independence assumption holds if and only if knowing the value of all variables but $s_i$ does not bring any information about $s_i$. 

Since $\mathcal{H}(S) \Rightarrow \mathcal{H}^*(S)$, we have $\neg \mathcal{H}^*(S) \Rightarrow \neg \mathcal{H}(S)$. I.e., showing that $\mathcal{H}^*$ does not hold is a condition sufficient but not necessary to show that $\mathcal{H}$ does not hold. Thus, we can define a relaxed degree of dependency $\mathcal{D}^*$ of a subspace $S$, as the discrepancy of the observed marginal distribution $p^o_{s_i}(X)$ 
and the expected one $p^{e}_{s_i}(X)$. Under the relaxed independence assumption $\mathcal{H}^*$, we have $p^{e}_{s_i}(X) = p^o_{s_i}(X|\overline{X_{s_i}}), \forall s_i \in S$. 
We define $\mathcal{D}^*$ as the expected value of those discrepancies:

\begin{dfn}[Relaxed Degree of Dependency]
	\label{def:RelaxedDegree}
	\begin{align} 
	\mathcal{D}^*(S) &\equiv \mathop{\mathbb{E}}_{s_i \in S}   \Big[disc \left( p^o_{s_{i}}(X) , p^o_{s_i}(X|\overline{X_{s_i}}) \right)  \Big]
	\end{align}
\end{dfn}

This definition is broad and contains a whole class of dependency estimators, e.g., \cite{Keller2012}. This class of estimators aims at measuring a so-called notion of  \textit{contrast} of the subspace. 
$\mathcal{D}^*$ -- or \textit{contrast} -- is a variant of $\mathcal{D}$ which is much easier to estimate: 
First, it relies on the comparison of marginal against conditional densities, i.e., multivariate density estimation is not required. 
Second, the number of degrees of freedom of $\mathcal{H}^*(S)$ increases linearly with $|S|$, while exponentially for $\mathcal{H}(S)$. Thus, $\mathcal{D}^*$ is  much less expensive to estimate than $\mathcal{D}$. 

By definition, $\mathcal{D}^*$ does not take into account the dependency between multivariate subsets, but only of each individual variable versus all others. However, we argue that this relaxation is not problematic. 
In fact, the detection of dependency is only interesting as long as we can observe effects w.r.t.\ the marginal and conditional distributions. In real-world scenarios, one is typically looking for interpretable influences of particular variables -- so-called `targets' -- on
the system and vice versa.

\subsubsection{Simulating Conditional Dependencies via Slicing}
\label{sec:slicing}

The main difficulty to estimate $\mathcal{D}^*$ is estimating the conditional distributions $p^o_{s_i}(X|\overline{X_{s_i}})$, because the underlying data distribution is unknown. As suggested in \cite{Keller2012}, we can simulate conditional distributions by applying a set of conditions \mbox{to $S$}, in a process called \textit{subspace slicing}.

\begin{dfn}[Subspace Slice]
	\label{slice}
	A subspace slice $c_{i}$ of $S$ is a set of $\left |S \right |-1$ conditions, where each condition is an interval $\left [l_{s_j}, u_{s_j} \right ]$, which restricts the values of $s_j \in S \setminus {s_i}$:  
	\begin{equation}
	\begin{split}
	&c_i = \left \{ \left[l_{s_j}, u_{s_j} \right] : s_j \in S\setminus s_i  \right \}~~s.t.~~\forall \left[l_{s_j}, u_{s_j}\right] \in c_i, \\
	& \left | \left \{\vec{x_k}: \vec{x_k} \in B \wedge x_k^{s_j}  \in \left [l_{s_j}, u_{s_j} \right ] \right \} \right | = n'
	\end{split} 
	\end{equation}
	where $n' \in \{1, \dots, n\}$ is the number of objects per condition, and $s_i$ is the \textbf{reference} dimension.
	We say that $\vec{x_k} \in c_i$ when $\vec{x_k}$ fulfils all the conditions in $c_i$. We define $\bar{c}_i$ as the set of complementary conditions of a given $c_i$: 
	\begin{align}
	\bar{c}_i = \left \{ (-\infty, l_{s_j}) \cup (u_{s_j}, \infty) : \left[l_{s_j}, u_{s_j}\right] \in c_i \right\}
	\end{align}
	$p_{s_i | c_i}(X)$ and $p_{s_i | \bar{c}_i}(X)$ denote the conditional \textit{pdf} of the observation in the slice $c_i$ and its complement $\bar{c}_i$ respectively. $\mathcal{P}^{c}(S)$ is the set of all possible slices in $S$.
\end{dfn}

We choose each interval in a slice at random and independently from each other. Under the independence assumption, the expected share of observations  $\alpha$ in the slice is equal to:
\begin{align}
\alpha = (n'/n)^{|S|-1} 
\end{align}

Interestingly, $n'$ can be determined given $\alpha$ as only exogenous parameter and the dimensionality $|S|$:
\begin{align}
n' = \left \lceil n \sqrt[|S|-1]{\alpha}\,  \right \rceil 
\end{align}

As a result, subspace slicing can be done in a \textbf{di\-men\-sio\-na\-li\-ty-aware} fashion. 
When $\alpha$ is a constant, the expected number of objects per slice does not change between subspaces with different dimensionality. 
One can see subspace slicing as a dynamic grid-based method, which does not suffer from the curse of dimensionality.

\begin{dfn}[Dimensionality-aware Slice]
	\label{def:dynamic-slice}
	A dimensionality-aware slice $c^\alpha_i$ of subspace $S$ is a set of $\left |S \right |-1$ conditions:
	\begin{align}
	c^\alpha_i = c_i~~s.t.~~n' = \left \lceil n \sqrt[|S|-1]{\alpha}\,  \right \rceil 
	\end{align}
\end{dfn}
For brevity, we assume a fixed $\alpha \in (0,1)$  and write $c_i = c_i^\alpha$, and we omit $(X)$ in $p_{s_j}(X)$ and  $p_{s_j|c_j}(X)$ in the following. 

The idea behind dimensionality-aware slicing is to simulate conditional distributions empirically. Under the $\mathcal{H}^*$-assumption, the conditional distribution $p_{s_i|c_i}$ is equal to the marginal distribution $p_{s_i}$, for any dimension $s_i$ and slice $c_i$. 
\begin{theorem}[$\mathcal{H}^*$ and Conditional Distributions]
	\label{theorem2}
	\begin{align}
	\mathcal{H}^*(S) \Leftrightarrow 
	p_{s_i | {c_i}} 
	&= 
	p_{s_i} 
	&\forall s_i \in S, \forall c_i \in \mathcal{P}^{c}(S) 
	\end{align}
\end{theorem}
\begin{proof}
	By contradiction, using Theorem \ref{IA2}.
	\begin{align*}
	\intertext{`$\Leftarrow$': From  Theorem \ref{IA2}, assume $\mathcal{H}^*(S)$ and that } 
	~& \exists s_j \in S : p_{s_j}(X|\overline{X_j}) \neq p_{s_j} \\
	&\Rightarrow  \exists c_j \in \mathcal{P}^c(S) : p_{s_j | c_j} \neq p_{s_j} \\
	&\Rightarrow \text{Contradiction of Theorem \ref{theorem2}}  \notag
	\intertext{`$\Rightarrow$': From Theorem \ref{theorem2}, assume $\mathcal{H}^*(S)$ and that }
	~& \exists s_j \in S, \exists c_j \in P^c(S) : p_{s_j | c_j} \neq p_{s_j} \\
	&\Rightarrow  p_{s_j}(X|\overline{X_{s_j}}) \neq p_{s_j} \\
	&\Rightarrow  \text{Contradiction of Theorem \ref{IA2}} \qedhere
	\end{align*}
\end{proof}

\subsubsection{Discrepancy Estimation} In reality, one only has access to a limited number of observations. Thus, one must quantify the discrepancy between empirical distributions. The basic idea is to use a statistical test $\mathcal{T}$: 
\begin{align}
disc \left ( 
\hat{p}_{s_i} , \hat{p}_{s_i | {c_i}}
\right ) \equiv 
\mathcal{T}\left( \hat{p}_{s_i}, \hat{p}_{s_i | {c_i}} \right )
\end{align}
However, since the number of observations is finite, the observations underlying $\hat{p}_{s_i | {c_i}}$
are \textbf{included} in the set of observations from $\hat{p}_{s_i}$. 
This is problematic, as statistical tests assume the two samples to be \textbf{distinct}. Plus, when $\alpha \approx 1$, $\hat{p}_{s_i | {c_i}}$ converges to $\hat{p}_{s_i}$, i.e., the two populations are nearly the same. Conversely, $\alpha \approx 0$ yields spurious effects, since the observations from $\hat{p}_{s_i | {c_i}}$ are then few. We solve the problem by observing that $p_{s_i | {c_i}}$ and $p_{s_i | {\bar{c}_i}}$ must be equal under $\mathcal{H}^*$. 

\begin{theorem}[$\mathcal{H}^*$ and Complementary Conditions] \label{theorem3}
	\begin{align}
	\mathcal{H}^*(S) \Leftrightarrow p_{s_i | {\bar{c}_i}} &= p_{s_i | {c_i}}
	&\forall s_i \in S, \forall c_i \in \mathcal{P}^{c}(S) 
	\end{align}
\end{theorem}

\begin{proof}
	By contradiction, using Theorem \ref{theorem2}.
	\begin{align*}
	\intertext{`$\Leftarrow$': From  Theorem \ref{theorem2}, assume $\mathcal{H}^*(S)$ and that}
	~ &\exists s_j \in S, \exists c_j \in P^c(S) : p_{s_j | c_j} \neq p_{s_j} \\
	\intertext{\quad since $p_{s_j} = p_{s_j | c_j \cup \bar{c}_j}$,}
	&\Rightarrow  \exists s_j \in S, \exists c_j \in P^c(S) : p_{s_j | c_j} \neq p_{s_j | c_j \cup \bar{c}_j} \\
	&\Rightarrow \exists s_j \in S, \exists c_j \in P^c(S) : p_{s_j | c_j} \neq p_{s_j | \bar{c}_j} \\
	&\Rightarrow  \text{Contradiction of Theorem \ref{theorem3}} \\
	\intertext{`$\Rightarrow$': From  Theorem \ref{theorem3}, assume $\mathcal{H}^*(S)$ and that} 
	~ &\exists s_j \in S, \exists c_j \in P^c(S) : p_{s_j | c_j} \neq p_{s_j | \bar{c}_j} \\ 
	&\Rightarrow \exists s_j \in S, \exists c_j \in P^c(S) : p_{s_j | c_j \cup c_j} \neq p_{s_j | \bar{c}_j \cup c_j} \\
	\intertext{\quad since $c_j \cup c_j = c_j$ and $p_{s_j} = p_{s_j | \bar{c}_j \cup c_j}$,}
	&\Rightarrow \exists s_j \in S, \exists c_j \in P^c(S) : p_{s_j | c_j} \neq p_{s_j} \\
	&\Rightarrow \text{Contradiction of Theorem \ref{theorem2}} \qedhere
	\end{align*}
\end{proof}

Hence, one can evaluate the $\mathcal{H}^*$-assumption by looking instead at the discrepancies between the conditional distribution and its complementary conditional distribution. When doing so, the samples obtained from both distributions are distinct. 

We have defined dimensionality-aware slicing based on $\alpha$, the expected share of observations in the slice $c_i$. Thus, the expected share of observations $\bar{\alpha}$ in $\bar{c}_i$ equals $1-\alpha$. This leads to setting $\alpha = 0.5$, so that $\bar{\alpha} = \alpha$. This choice is pertinent for statistical testing, as equal sample sizes lead to higher statistical stability, and we get rid of parameter $\alpha$. 

We also propose to restrict the domain of the reference dimension $s_i$ to the same proportion $\alpha$ of objects:

\begin{dfn}[Marginal Restriction]
	A marginal restriction is a condition on the reference dimension $s_i$, i.e., an interval $r_i:[l_{s_i}, u_{s_i}]$, such that $|\{ \vec{x_j} : x^{s_i}_j \in B \wedge x^{s_i}_j \in r_i | = \left \lceil \alpha \cdot n \right \rceil$. We define $p_{s_i | {c_i} | r_i}$ as the restricted conditional distribution given $c_i$, $r_i$. $\mathcal{P}^{r}(S)$ is the set of all restrictions.  
\end{dfn}

With the marginal restriction, the approach becomes more sensitive to local effects in the marginal distribution, compared to simply considering the full range. Furthermore, this reduces the number of points in the two samples by $\alpha$, leading to lower computational requirements of the underlying statistical test. 

\begin{figure}
	\hfill
	\begin{subfigure}{0.329\linewidth}
		\begin{center}
			\includegraphics[width=\linewidth]{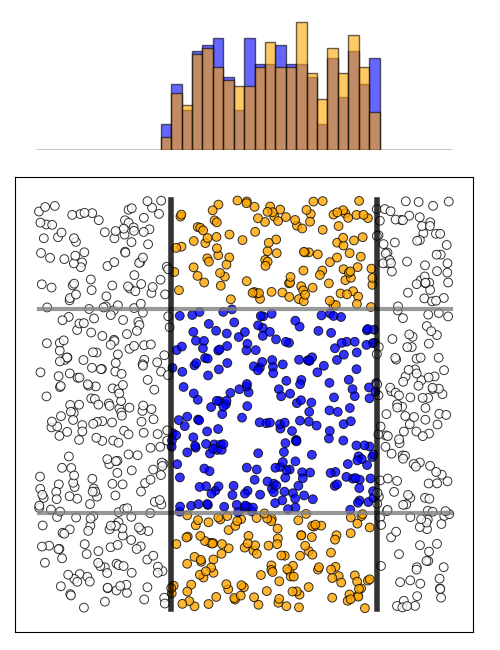}
		\end{center}
		\caption{Independent}
		\label{fig:slicing-independent}
	\end{subfigure}\hfill%\par%\medskip
	\begin{subfigure}{0.329\linewidth}
		\begin{center}
			\includegraphics[width=\linewidth]{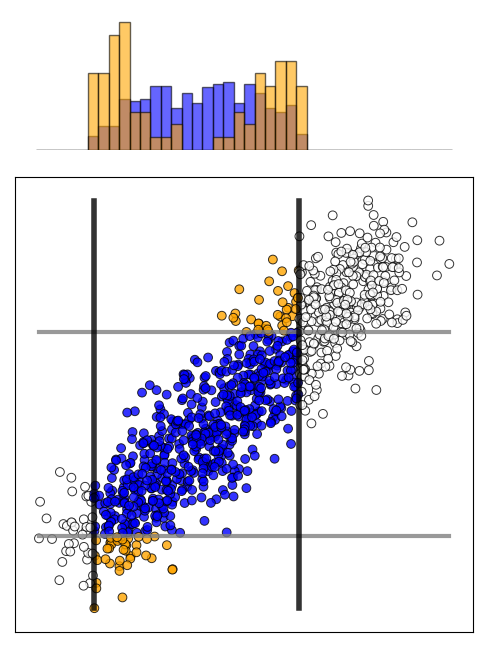}%\hfill
		\end{center}
		\caption{Linear}
		\label{fig:slicing-linear}
	\end{subfigure}\hfill
	\begin{subfigure}{0.329\linewidth}
		\begin{center}
			\includegraphics[width=\linewidth]{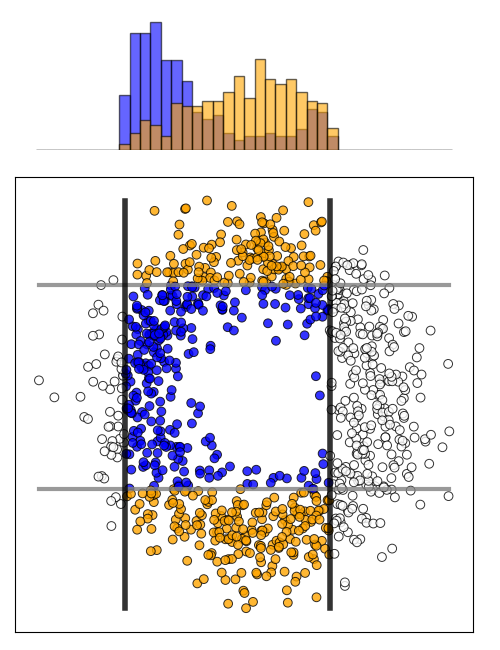}%\hfill
		\end{center}
		\caption{Circle}
		\label{fig:slicing-circle}
	\end{subfigure} \hfill
	\caption{Slicing in 2-D subspaces, with $\alpha =  0.5$} 
	\label{fig:slicing}
\end{figure}

We illustrate slicing in Figure \ref{fig:slicing}, with an independent subspace on the left-hand side and with subspaces with noisy dependencies on the right. The grey lines show a random slice $c_x$ on the $y$-axis. Two black bold lines stand for the restriction $r_x$. The points in dark blue are in the restricted slice $c_x | r_x$ and the points in light orange are in $\bar{c}_x | r_x$. Using histograms, we plot along the $x$-axis the distribution of the points in both samples. From the histograms, we can see that the two distributions are relatively similar for Figure \ref{fig:slicing-independent}, while they are markedly different for Figure \ref{fig:slicing-linear} and \ref{fig:slicing-circle}.

In the end, after each slicing operation, one obtains two object samples $B_{c_i | r_i }$ and $B_{\bar{c}_i | r_i }$ such that $B_{c_i | r_i } \cap  B_{\bar{c}_i | r_i } = \emptyset$, and we use a statistical test $\mathcal{T}$ to estimate the discrepancy between $\hat{p}_{s_i | {c_i} | r_i }$ and $\hat{p}_{s_i | {\bar{c}_i} | r_i }$. 

A statistical test $\mathcal{T} \left( B_1, B_2 \right)$ on two samples $B_1$ and $B_2$ typically yields a $p$-value. Traditionally, one uses $p$-value to assess the \textit{statistical significance}. It is the probability to falsely reject a true null hypothesis, where the null hypothesis is the independence. Conversely, $p^c = 1-p$ is a \textit{confidence level} or probability to truly reject a false null hypothesis. The rationale behind $\mathcal{D}^*$ is to yield values quantifying the independence violation.
We define our own notion of \textit{contrast}, abbreviated as $\mathcal{C}$, as the expected value of the \textit{confidence level} of a statistical test $\mathcal{T}$ between the samples from the conditional distributions for all the possible dimensions $s_i$, slices $c_i$ and restrictions $r_i$: 

\begin{dfn}[Contrast $\mathcal{C}$]
	\label{def:contrast-test}
	\begin{align} 
	\mathcal{C}(S) &\equiv
	\mathop{\mathbb{E}}_{\{c_i,r_i\} \in {\mathcal{P}^{c}\times}\mathcal{P}^{r}}
	\Big [ \mathcal{T} \left (B_{c_i | r_i }, B_{\bar{c}_i | r_i } \right ) \Big ]
	\end{align} 
	where $\mathcal{T}$ yields $p^c$-values, and we draw $c_i$, $r_i$ randomly and independently from each other w.r.t. any dimension $s_i \in S$.
\end{dfn}
By definition, and independently from the underlying test, $\mathcal{T} \sim \mathcal{U}[0,1]$ when the two samples are independent from each other, and $\mathcal{T} \approx 1$ as the evidence against independence becomes stronger. The properties of $\mathcal{C}$ follow:
\begin{itemize}
	\item $\mathcal{C}$ converges to $1$ as the dependency strength in $S$ increases, since the $p^c$-values converge stochastically to $1$.
	\item $\mathcal{C}$ converges to $0.5$ when $S$ is independent, since the distribution of the $p^c$-values converges to $\mathcal{U}[0,1]$.  
	\item $\mathcal{C}$ is bounded between $0$ and $1$, as the $p^c$-values are bounded between $0$ and $1$. 
\end{itemize}

\subsubsection{Monte Carlo Approximation}
\label{montecarlosimulation}
Unfortunately, $\mathcal{C}$ is impossible to compute exactly. Namely, one would need to know the distribution of $B_{c_i | r_i }$ and $B_{\bar{c}_i | r_i }$ for every dimension, slice and restriction. 
Instead, the idea is to approximate $\mathcal{C}$ via Monte Carlo (MC) simulations, using $M$ iterations. For each iteration, we choose the reference dimension, slice and restriction at random. The \textit{approximated contrast} $\mathcal{\hat{C}}$ is defined as follows:  

\begin{dfn}[Approximated Contrast $\mathcal{\hat{C}}$]
	\label{def:contrast-approx}
	\begin{align} 
	\mathcal{\hat{C}}(S) &= 
	\frac{1}{M} \sum_{m=1}^{M} \mathcal{T} \left ( B_{ \left [c_i | r_i \right ]_m }, B_{\left [\bar{c}_i | r_i \right ]_m } \right )
	\end{align} 
	where $\left [c_i | r_i \right ]_m$ means that we draw $i$, $c_i$ and $r_i$ randomly at iteration~$m$, i.e., $i \leftarrow \{1,...,|S|\}$ and $\{c_i, r_i\} \leftarrow \mathcal{P}^{c} \times \mathcal{P}^{r}$. 
\end{dfn}

\label{bound} Interestingly, we can bound the quality of the approximation. From Hoeffding's inequality \cite{Hoeffding1963}, we derive a bound on the probability of $\mathcal{\hat{C}}$ to deviate not less than $\varepsilon$ from $\mathcal{C}$. The bound decreases exponentially with increasing $M$:

\begin{theorem}[Hoeffding's Bound of $\mathcal{\hat{C}}$]
	\label{"th:hoeffding-chernoff-contrast"}
	\begin{align}
	\Pr\left(| \mathcal{\hat{C}} - \mathcal{C} | \geq \varepsilon \right) \leq 2e^{-2M \varepsilon^2}
	\end{align}
	where $M$ is the number of MC iterations, and $0 < \varepsilon < 1 - \mathcal{C}$. 
\end{theorem} 

\begin{proof}
	Let us first restate Theorem 1 from Hoeffding \cite{Hoeffding1963}:
	Let $X_1, X_2, \dots, X_n$ be independent random variables $0 \leq X_i \leq 1$ for $i \in \{1,\dots,n\}$ and let $\bar{X} = \frac{1}{n}(X_1 + X_2 + \dots + X_n)$ be their mean with expected value $E[\bar{X}]$. Then, for $0<t<1-E[\bar{X}]$: 
	\begin{align*}
	\Pr\left(\bar{X}-E[\bar{X}] \geq t \right) \leq e^{-2nt^2} && \Pr\left(\bar{X}-E[\bar{X}] \leq t \right) \leq e^{-2nt^2}
	\end{align*}
	We can treat each MC iteration $m_1, m_2, \dots, m_M$ as i.i.d. random variables $X_{m_1}, X_{m_2}, \dots, X_{m_M}$ in $[0,1]$ with mean $\mathcal{\hat{C}}$ and expected value $E[\mathcal{\hat{C}}] = \mathcal{C}$ (\mbox{Definition \ref{def:contrast-test}}). Thus, for $0 < \varepsilon < 1 - \mathcal{C}$, we have $\Pr (| \mathcal{\hat{C}} - \mathcal{C} | \geq \varepsilon) \leq 2e^{-2M \varepsilon^2}$. 
\end{proof}

This is very useful. For instance, when $M=200$, the probability of $\mathcal{\hat{C}}$ to deviate more than $0.1$ from its expected value is less than $2e^{-4} \approx 0.04$, and this bound decreases exponentially with $M$. Thus, one can adjust the computational requirements of $\mathcal{\hat{C}}$, given the available resources or a desired quality level. In other words, users can set $M$ intuitively, as it leads to an expected quality, and vice versa. Furthermore, $M$ is the only parameter of the \textit{MCDE} framework.

\subsection{Instantiation as \textit{MWP}}
\label{instantiation-as-mwp}

To use the \textit{MCDE} framework, one must instantiate a suitable statistical test as $\mathcal{T}$. 
To comply with our requirements, this statistical test needs to be \textbf{non-parametric} (\textbf{R5}) and \mbox{\textbf{\textbf{robust}} (\textbf{R8})}. As a proof of concept, we instantiate $\mathcal{T}$ as a two-sided \emph{Mann-Whitney U} test \cite{Mann1947}, abbreviated as \emph{U} hereafter.

The \emph{U} test has the following features which other statistical tests may lack.
First, it is one of the most powerful statistical tests \cite{Mood1954}: Its power-efficiency approaches 95.5\% when
comparing it to the $t$-test, as the number of observations $n$ increases. But contrary to the $t$-test, the \emph{U} test  is \textbf{non-parametric}. Second, \cite{Dixon1954} shows that the \emph{U} test is more efficient than the \textit{Kolmogorov-Smirnov} test for large samples. 
Finally, the \emph{U} test  does not require continuous data, as it operates on ranks. Thus, it is \textbf{robust} and applicable to virtually any kind of ordinal measurements. 

We review the definition of the \textit{U} test \cite{Siegel1956} between two samples $B_1$ and $B_2$ with size $n_1$ and $n_2$. It tests the null hypothesis that it is equally likely that a randomly selected value from one sample will be less than or greater than a randomly selected value from the other sample. 
$R_1$ and $R_2$ are the sums of ranks of the objects in $B_1$ and $B_2$, obtained by ranking the values of $B_1$ and $B_2$ together, starting with $0$. In case of ties, the ranks of the tying objects are \textit{adjusted}, i.e., become the average of their ranks. 
\begin{align}
\textit{U test} :&&  p = \Phi(Z)~\text{or}~1-\Phi(Z) && Z = \frac{ U - \mu }{\sigma}
\end{align}
Here, one can choose $U = U_1$ or  $U = U_2$ equivalently, with:
\begin{align}
U_1 = R_1 - \frac{n_1(n_1-1)}{2} && U_2 = R_2 - \frac{n_2(n_2-1)}{2}
\end{align}
$\Phi$ is the cumulative distribution function (\textit{cdf}) of the normal distribution; 
$\mu$, $\sigma$ are defined as: 
\begin{align}
\mu = \frac{n_1 n_2}{2} && \sigma = \sqrt{\frac{n_1 n_2}{12} \left ( (n+1) - \sum_{i=1}^{k} \frac{t_i^3 - t_i}{n(n-1)} \right ) }
\end{align}
The summation term of $\sigma$ is a correction for ties, where $t_i$ is the number of observations sharing rank $i$, and $k$ is the number of distinct ranks.
For large enough samples, typically $n > 30$, the values of $U$ are normally distributed \cite{Mann1947} with mean $\mu$ and standard deviation $\sigma$. $Z$ is the standardized score, since \mbox{$Z \sim \mathcal{N}(0,1)$}. If $U = U_1$, then $Z \ll 0$ and $p \approx 0$ when the ranks of $A_1$ are stochastically smaller than those of $A_2$. Conversely, when the ranks of $A_1$ are stochastically larger, then $Z \gg 0$ and $p \approx 1$.
Both cases indicate an independence violation. As both directions are relevant, our test should capture them equally. 

We implement a two-sided version of the \textit{Mann-Whitney U} test, which we dub $\mathcal{T}_{\textit{MWP}}$. The letter \textit{P} emphasises that the test returns a $p^c$-value, as required by the \textit{MCDE} framework:
\begin{align}
\mathcal{T}_{\textit{MWP}}: && {p^c}= \Phi^{1/2}(Z') && Z' = |Z| && U = U_1
\end{align}

Since $Z  \sim \mathcal{N}(0,1)$, $Z'$ follows the so-called half-normal distribution with \textit{cdf} $\Phi^{1/2}$. Since $|U_1 - \mu|$ = $|U_2 - \mu|$, we can simply set $U = U_1$ or $U = U_2$ arbitrarily, i.e., one only needs to sum the ranks of one of the samples.

In the end, when the independence assumption does not hold, we expect the ranks of the two samples after slicing to differ, which leads to $\mathcal{T}_{\textit{MWP}} \approx 1$.
Thus, the test complies with Definitions \ref{def:contrast-test} and \ref{def:contrast-approx}. 
We define \textit{MWP} or $\mathcal{\hat{C}}_{\textit{MWP}}$ as the instantiation of $\mathcal{\hat{C}}$ using $\mathcal{T}_{\textit{MWP}}$ as statistical test: 

\begin{dfn}[Mann-Whitney P (\textit{MWP})]
	\begin{align}
	\text{\textit{MWP}} = \mathcal{\hat{C}}_{\textit{MWP}} = \frac{1}{M} \sum_{m=1}^{M}
	\mathcal{T}_{\textit{MWP}} \left (B_{\left [c_i | r_i \right ]_m }, B_{ \left [\bar{c}_i | r_i \right ]_m } \right )
	\end{align}
\end{dfn}

\subsection{Algorithmic Considerations \& Complexity} 

We now outline our algorithm to efficiently compute \textit{MWP}.

\subsubsection{Computing an Index Structure} \label{computing-index}

The \textit{MCDE}-\textit{MWP} approach requires the creation of an index, as a preprocessing step, to avoid the expensive repetition of sorting operations. The index $\mathcal{I}$ is a one-dimensional structure containing the adjusted ranks and tying values corrections for each dimension.
It  consists of $|S|$ elements $\{I_1, \dots, I_{|S|} \}$, where $I_i$ is an array of $3$-tuple $[(l_1^i, a_1^i, b_1^i),\dots,(l_n^i, a_n^i, b_n^i)]$ ordered by $s_i$ in ascending order. 
In this tuple, $l^i$ are the row numbers of the values of $s_i$, $a^i$ are the \textit{adjusted} ranks and $b^i$ the accumulated correction of the standard deviation $\sigma$ from the first element. We denote $I_i[j]$, $s_i[j]$ as the $j$-th elements of $I_i$ and $s_i$; we refer to the components of $I_i[j]$ as $l_j^i$, $a_j^i$, $b_j^i$.
We outline the construction of the index in Algorithm \ref{indexconstruction}. 
For each attribute $s_i$, we sort the values (Line~$4$) and perform a single pass over the sorted list to adjust the ranks and the correction for ties. Thus, the index construction complexity is in $O(|S| \cdot (n \cdot log(n) + n))$.

\begin{algorithm}
	\caption{MWP Index Construction}\label{indexconstruction}
	
	\begin{algorithmic}[1]
		\Function{ConstructIndex}{$S = \{s_i\}_{i \in \{1,\dots,d\}}$}
		\For{$i = 1$ to $|S|$} 
		
		\State $r^i \gets \left[0,\dots,n-1\right]$
		\State $l^i \gets$ sort $r^i$ by $s_i$ in ascending order
		\State $I_i \gets \left[(l^i_1, r^i_1),\dots,(l^i_n, r^i_n)\right]$ \Comment{Initialize $I_i:(l^i, r^i)$}
		\State $j \gets 1$ ; $correction \gets 0$
		\While{$j \leq n$}
		\State $k \gets j$ ; $t \gets 1$ ; $adjust \gets 0$
		\While{$(k < n-1) \wedge (s_{i}[l^i_k] = s_{i}[l^i_{k+1}])$}
		\State $adjust \gets adjust + r^i_{k}$
		\State increment $k$ and $t$
		\EndWhile
		\If{$k > j$} \Comment{Adjust the rank and correction}
		\State $adjusted \gets (adjust + r^i_{k}) / t$
		\State $correction \gets correction + t^3 - t$ 
		\For{$m \gets j$ to $k$}
		\State $I_i[m] \gets (l_m^i, adjusted, correction)$
		\EndFor
		\Else $~~I_i[j] \gets (l_j^i, r_j^i, correction)$
		\EndIf
		\State $j \gets j + t$
		\EndWhile
		
		\EndFor
		\State \Return $\mathcal{I}: \{I_1, \dots, I_{|S|} \}$ with $I_i: (l^i, a^i, b^i)$
		\EndFunction
	\end{algorithmic}
\end{algorithm}

\subsubsection{Slicing over the Index Structure}

We can slice the input data efficiently, because the tuples are already sorted in the index structure. 
We successively mask the row numbers based on a random condition for all but one reference attribute $s_r$. 
Algorithm \ref{slicing} is the pseudocode of the slicing process. 
The complexity of slicing is in  $O(|S| \cdot n)$.

\begin{algorithm}
	\caption{Dynamic Slicing}\label{slicing}
	\begin{algorithmic}[1]
		\Function{Slice}{$\mathcal{I}: \{I_1, \dots. I_{|S|} \}$, $r$}
		\State $slice \gets$ Array of $n$ boolean values initialized to $\textit{true}$
		\State $slicesize \gets \left \lceil n \cdot \sqrt[|S|-1]{\alpha}\,  \right \rceil  $
		\For{$I_i \in \mathcal{I} \setminus I_{r}$} 
		\State $start \gets$ random integer in $[1,n-slicesize]$
		\State $end \gets$ $start + slicesize$
		\For{$j \gets 1$ until $start$ and $end+1$ to $n$} 
		\State $slice[l_j^i] \gets \textit{false}$
		\EndFor
		\EndFor
		\State \Return $slice$
		\EndFunction
	\end{algorithmic}
\end{algorithm}

\subsubsection{Computing the Statistical Test} We give the pseudocode to compute the statistical test in Algorithm \ref{TMWP}. We determine a restriction $[start, end]$ on $s_{r}$ and sum the adjusted ranks of the objects that belong to the slice. 
Thanks to the marginal restriction, we compute the statistical test in a subsample of size $n' < n$. Since the ranks in this subset may not start from $0$, we adjust the sum of the ranks $R_1$ (Line $11$). Then, we compute a correction term (Line $14$) using the cumulative correction $b^r$ to adjust $\sigma$ for ties (Line $15$).
We compute the statistical test via a single pass, considering only elements between $start$ and $end$. Each operation requires constant time, thus, the complexity of this step is in $O(n)$.

\begin{algorithm}
	\caption{$\mathcal{T}_{\textit{MWP}}$}\label{TMWP}
	\begin{algorithmic}[1]
		\Function{Compute $\mathcal{T}_{\textit{MWP}}$}{$\mathcal{I}: \{I_1, \dots. I_{|S|} \}$, $slice$, $r$}
		
		\State $start \gets$ random integer in $[1,n \cdot (1-\alpha)]$ 
		\State $end \gets$ $start + \left \lceil n \cdot \alpha \right \rceil$
		
		\State $R_1 \gets 0$ ; $n_1 \gets 0$
		\For{$j \gets start$ to $end$} 
		\If{$slice[l^r_{j}] = \textit{true}$}
		\State $R_1 \gets R_1 + a^r_{j}$
		\State $n_1 \gets n_1 + 1$
		\EndIf
		\EndFor
		\State{$n' \gets end - start $} %
		\If{$n_1 = 0$ or $n_1 = n'$}
		 \Return $1$%
		\EndIf
		\State $U_1 \gets R_1 - start \cdot n_1$ 
		\State $n_2 \gets n' - n_1$
		\State $\mu \gets (n_1 \cdot n_2)/ 2$
		\State $correction \gets (b^r_{end-1}-b^r_{start-1})/(n'\cdot (n'-1))$ 
		\State $\sigma \gets \sqrt{((n_1 \cdot n_2)/12) \cdot {(n'+1-correction)}}$
		\State \Return $\Phi^{1/2}(|U_1 - \mu| / \sigma)$
		\EndFunction
	\end{algorithmic}
\end{algorithm}

\subsubsection{Computing \textit{MWP}} \label{ComputingMWP} To determine $\mathcal{\hat{C}}_{\textit{MWP}}$ 
of a subspace $S$, we first construct  the index. Then, in $M$ iterations, we slice the data set (Algorithm \ref{slicing}) and compute $\mathcal{T}_{\textit{MWP}}$ (Algorithm \ref{TMWP}). \textit{MWP} is the average of $\mathcal{T}_{\textit{MWP}}$. See Algorithm \ref{MWP_alg}. 

\begin{algorithm}
	\caption{\textit{MWP}}\label{MWP_alg}
	\begin{algorithmic}[1]
		\Function{\textit{MWP}}{$S = \{s_i\}_{i \in \{1, \dots, d\} }$}
		\State $\mathcal{I} \gets$ \Call{ConstructIndex}{$S$} ; $\textit{MWP} \gets 0$
		\For{$m \gets 1$ to $M$}
		\State $r \gets$ random integer in $\left[ 1,d \right]$ 
		\State $slice \gets$ \Call{Slice}{$\mathcal{I}$, $r$}
		\State $\textit{MWP} \gets \textit{MWP } +$ \Call{Compute $\mathcal{T}_{\textit{MWP}}$}{$\mathcal{I}$ , $slice$, $r$}
		\EndFor
		\State $\textit{MWP} \gets \textit{MWP} / M$
		\State \Return \textit{MWP}
		\EndFunction
	\end{algorithmic}
\end{algorithm}

When one wants to replace the statistical test, one only needs to change Algorithm \ref{TMWP}. The rest is part of the \textit{MCDE} framework and does not require any adaptation. 

\subsubsection{Complexity}
The overall complexity of \textit{MWP} is in $O(|S| \cdot (n \cdot log(n) + n) + M \cdot(|S|\cdot n + n))$. Since $|S| \ll n$, this simplifies to $O(n \cdot log(n) + M\cdot n)$. The index construction is asymptotically the most expensive step, as it is in $O(n \cdot log(n))$. However, one only needs to construct the index once. When the index for a given data set is available, one can compute \textit{MWP} in linear time for the exponential number of subspaces. 

Interestingly, \textit{MWP} is trivial to parallelise: one can compute the elements of the index structure $I_1, \dots, I_{|S|}$ in parallel, as they are independent from each other. Similarly, one can parallelise each Monte Carlo iteration in the loop at Line 3 of Algorithm \ref{MWP_alg}. This is useful, as multi-core architectures are ubiquitous in modern database systems. 

Thus, \textit{MWP} scales well with the size of the data set. We will verify this claim via experiments in Section \ref{scalability}.

\section{Evaluation}

\label{evaluation}

In this section, we show via experiments that \textit{MWP} fulfils our requirements from Section~\ref{motivation}.
We also compare our approach to a range of state-of-the-art dependency estimators, namely \textit{MS}, \textit{II}, \textit{TC}, \textit{CMI}, \textit{MAC}, \textit{UDS} and \textit{HiCS}.

We implement \textit{MWP} in Scala, while other approaches are implemented in Java. Note that, fundamentally, the impact of this difference on runtime is low, as both run in the JVM.
We re-implement \textit{MS} following \cite{Schmid2007}, \textit{TC} and \textit{II} following \cite{Timme2014} using Kraskov's \cite{Kraskov2004} and Kozachenko \& Leonenko estimators respectively, with parameter $k=4$. We use the R*-tree implementation from ELKI \cite{DBLP:journals/pvldb/SchubertKEZSZ15} to increase the efficiency of nearest neighbour queries. For \textit{CMI}, \textit{MAC}, \textit{UDS} and \textit{HiCS}, we use the implementation provided in \cite{Nguyen2015UDS}. Each algorithm runs single-threaded in a server with 32 cores at 2.2 GHz and 64GB RAM. We use the default parameters, if any. If not stated otherwise, the data samples we use in our experiments have the size $n = 1000$, $d = 3$, and we set $M = 50$ for \textit{MWP} and \textit{HiCS}. In most existing studies, such as in  \cite{Nguyen2014MAC, Reshef2011},
$n$ usually is equal or lower. 

\subsection{Methodology} 

To compare the approaches, the idea is to characterise the distribution of the score they produce w.r.t.\ different dependencies of variable strength and noise. Intuitively, stronger dependencies should lead to higher scores than noisier ones. 
\subsubsection{Dependency Generation} 
For benchmarking, we use an assortment of 12 multivariate dependencies scaled to $[0,1]$. 
Figure \ref{fig:deps-plot} represents each of them in two and three dimensions. 
For each dependency, we repeatedly draw $n$ objects with $d$ dimensions, to which we add Gaussian noise with standard deviation $\sigma$, which we call \textit{noise level}. 

\newcommand{\depwa}{0.32}
\begin{figure}
	\centering
	\begin{subfigure}[b]{\depwa\linewidth}
		\hfill
		\includegraphics[width=0.32\linewidth]{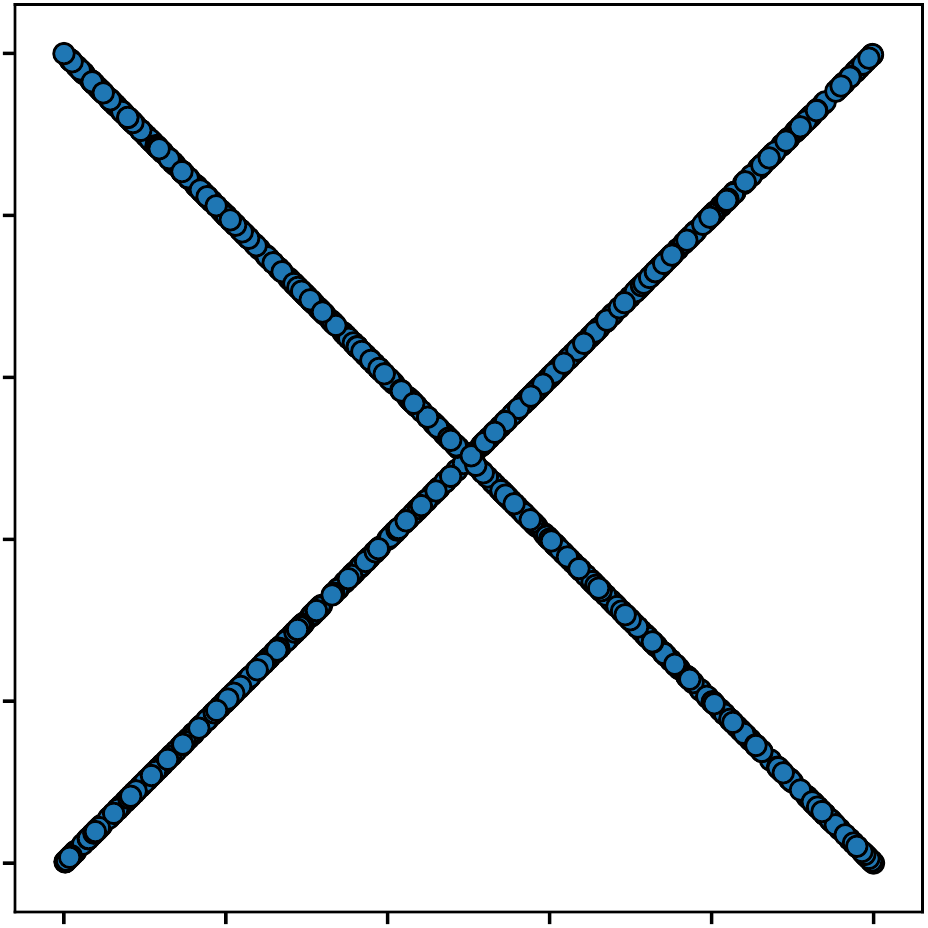} 
		\hfill
		\includegraphics[width=0.35\linewidth]{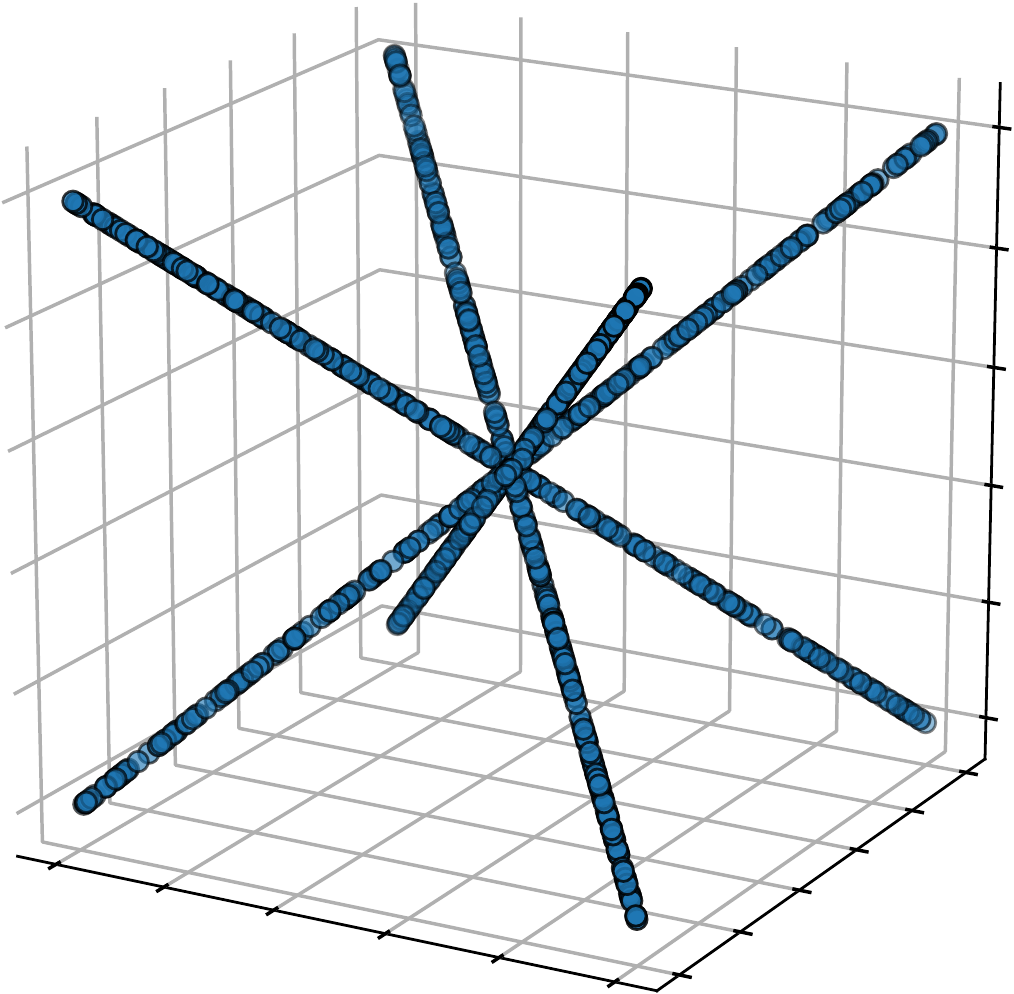}
		\hfill
		\caption{Cross (C)} 
	\end{subfigure} %
	\hfill
	\begin{subfigure}[b]{\depwa\linewidth}
		\hfill
		\includegraphics[width=0.32\linewidth]{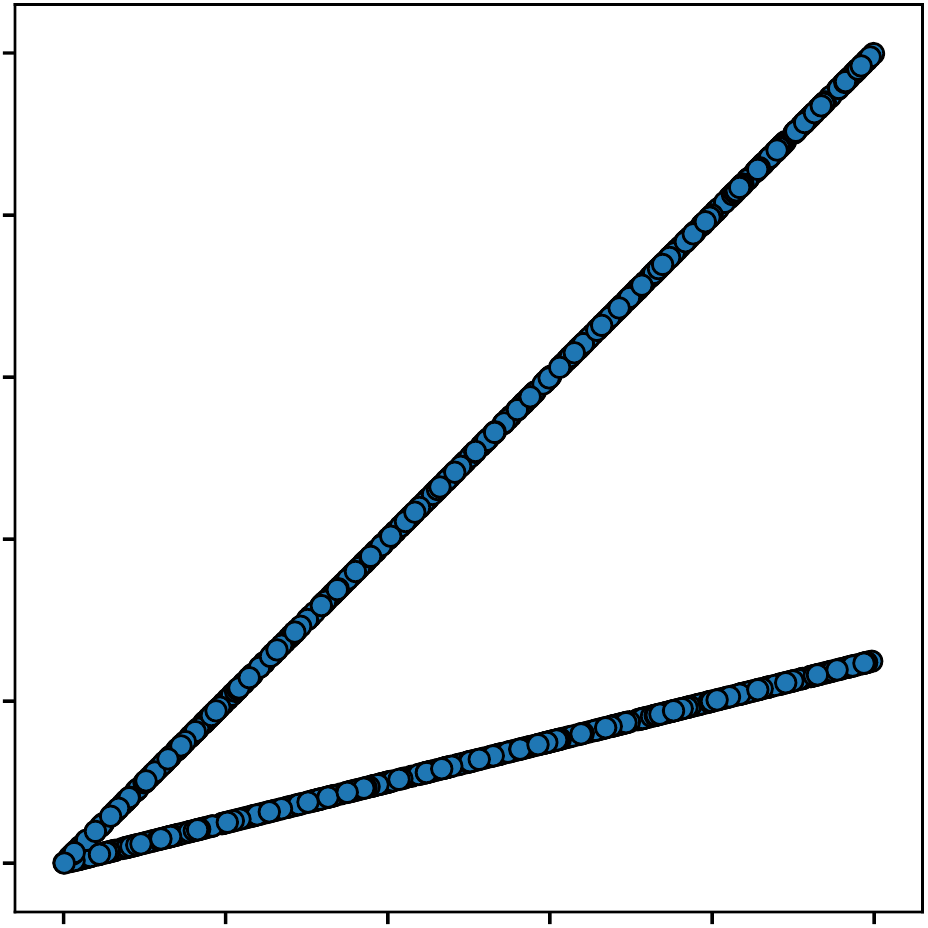} 
		\hfill
		\includegraphics[width=0.35\linewidth]{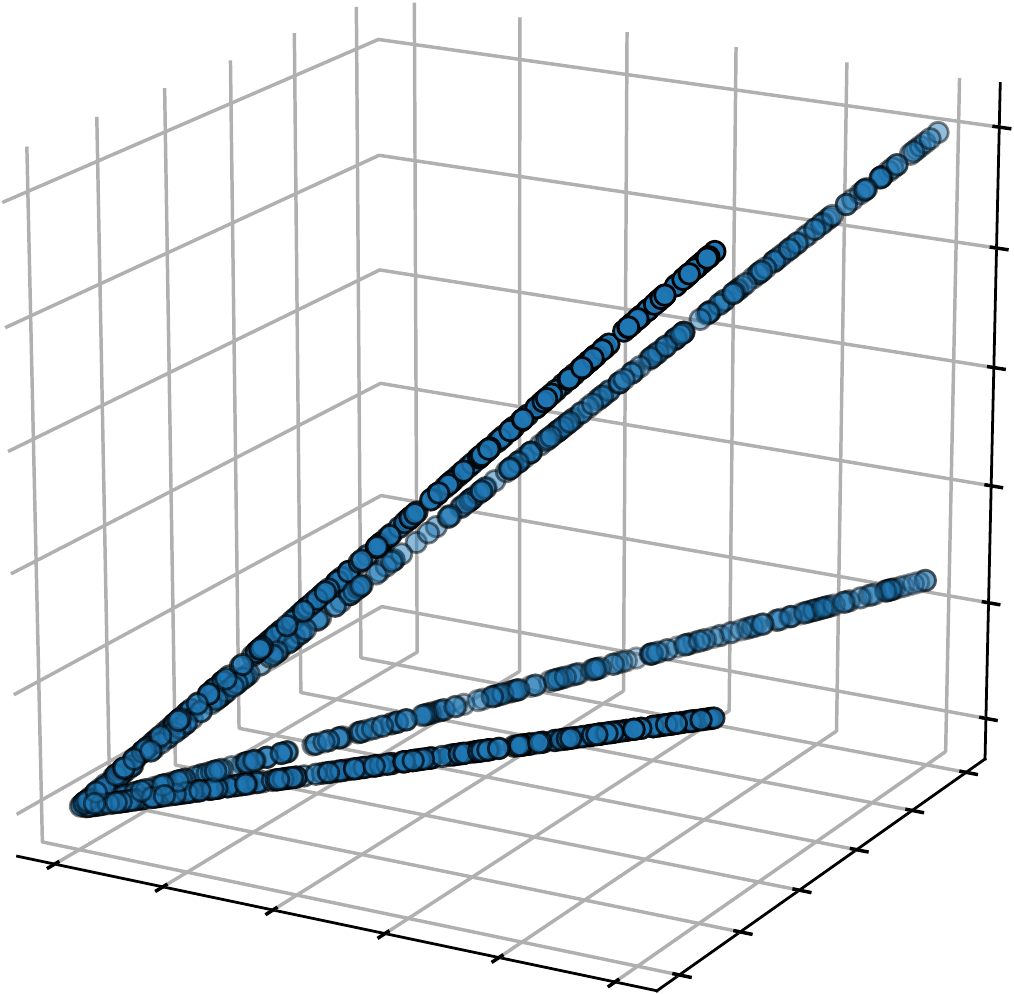} 
		\hfill
		\caption{Double linear (Dl)}
	\end{subfigure}
	\hfill
	\begin{subfigure}[b]{\depwa\linewidth}
		\hfill
		\includegraphics[width=0.32\linewidth]{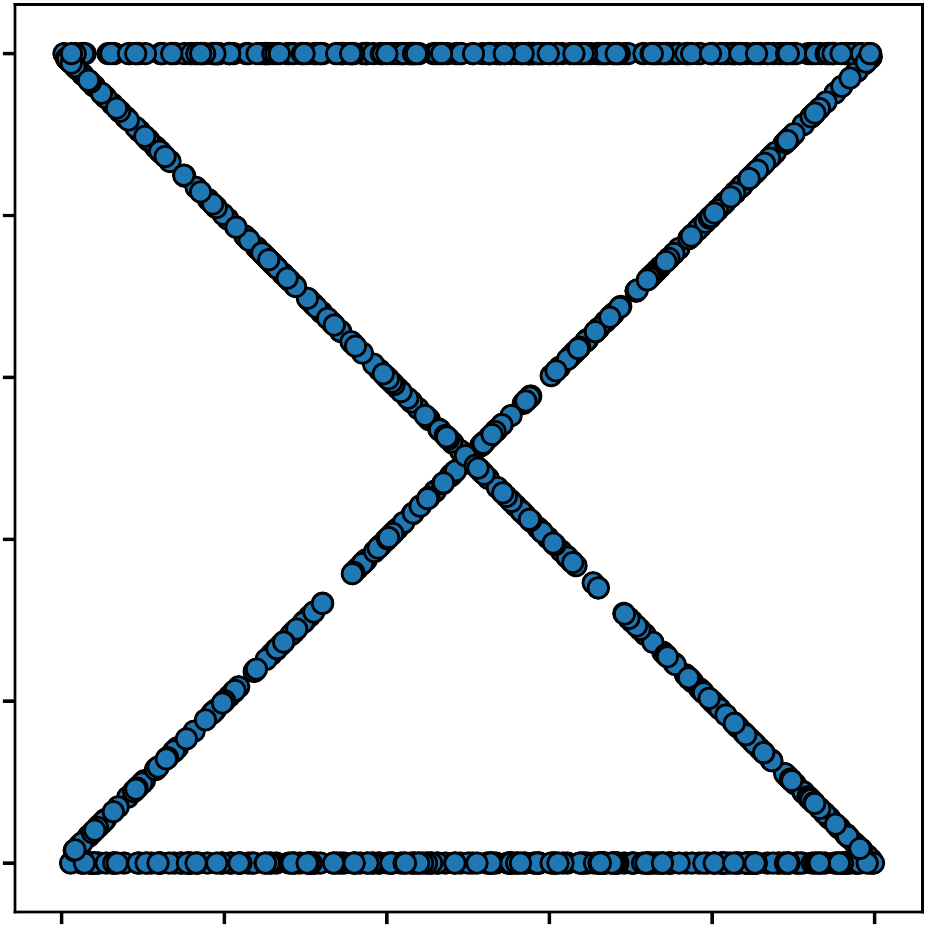} 
		\hfill
		\includegraphics[width=0.35\linewidth]{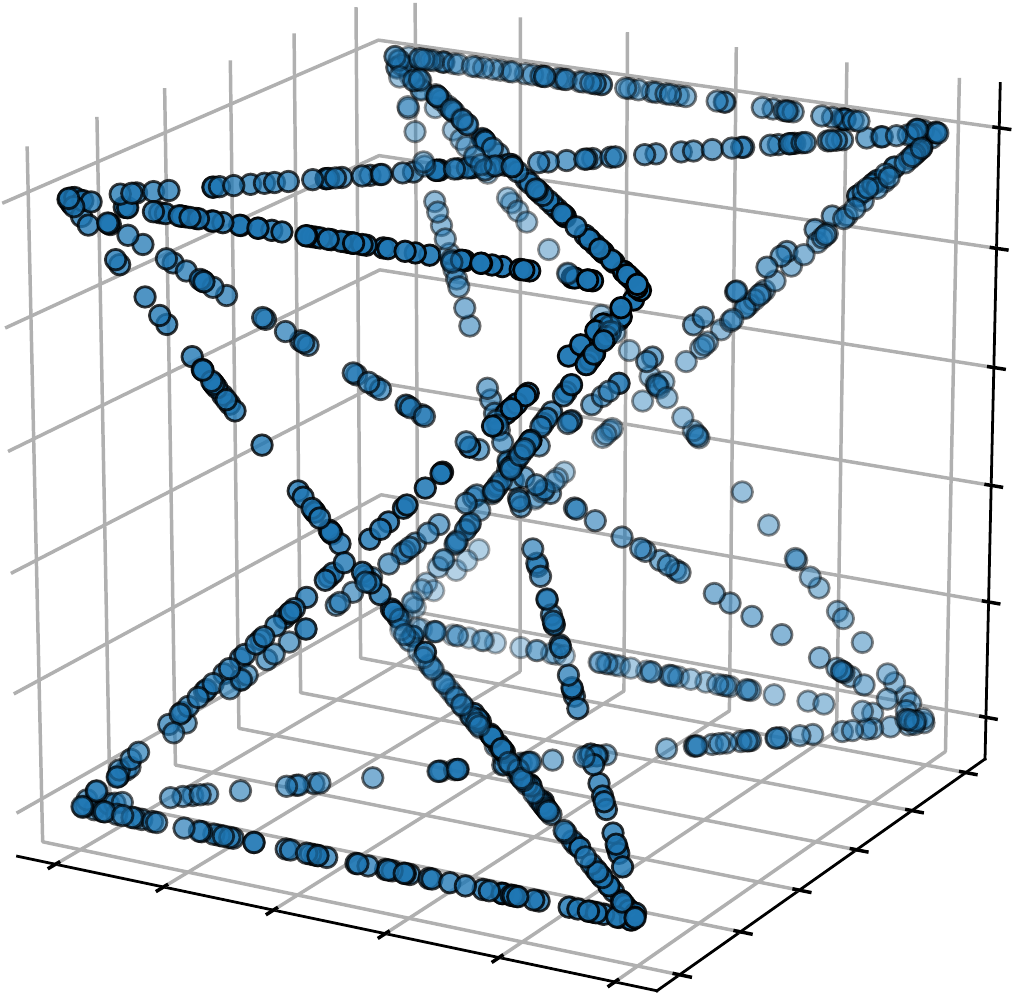} 
		\hfill
		\caption{Hourglass (H)}
	\end{subfigure}
	
	~\\
	\begin{subfigure}[b]{\depwa\linewidth}
		\hfill
		\includegraphics[width=0.32\linewidth]{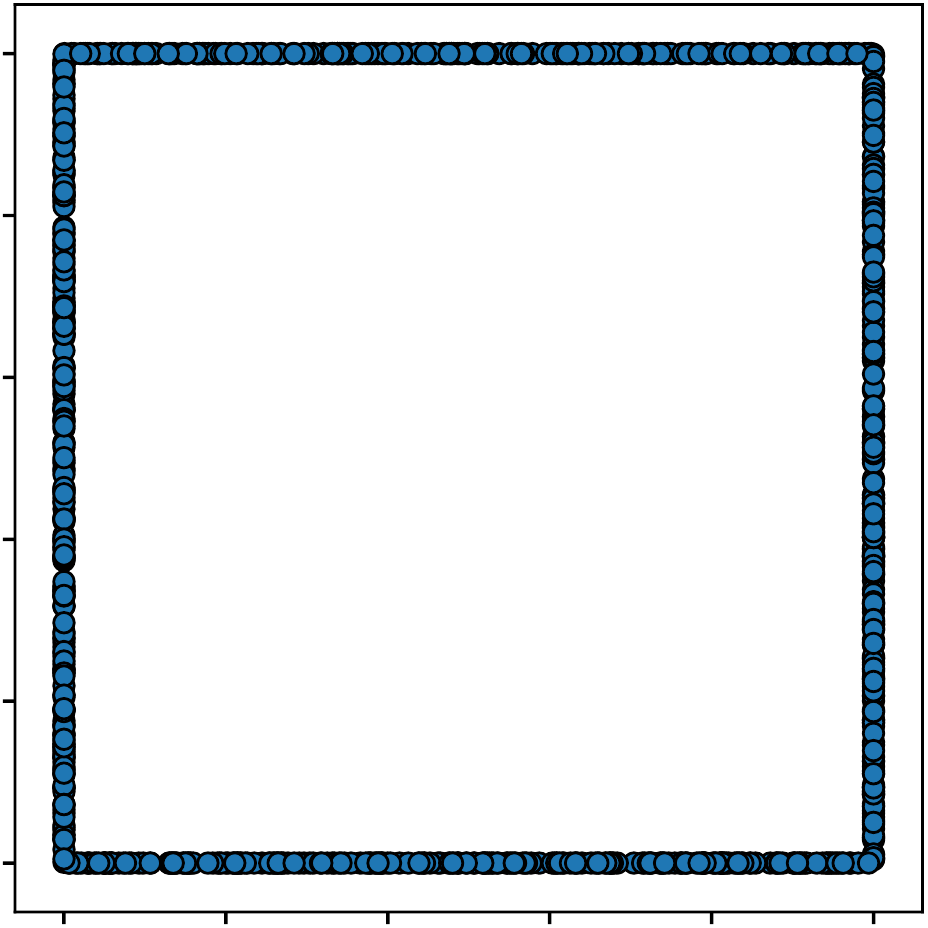} 
		\hfill
		\includegraphics[width=0.35\linewidth]{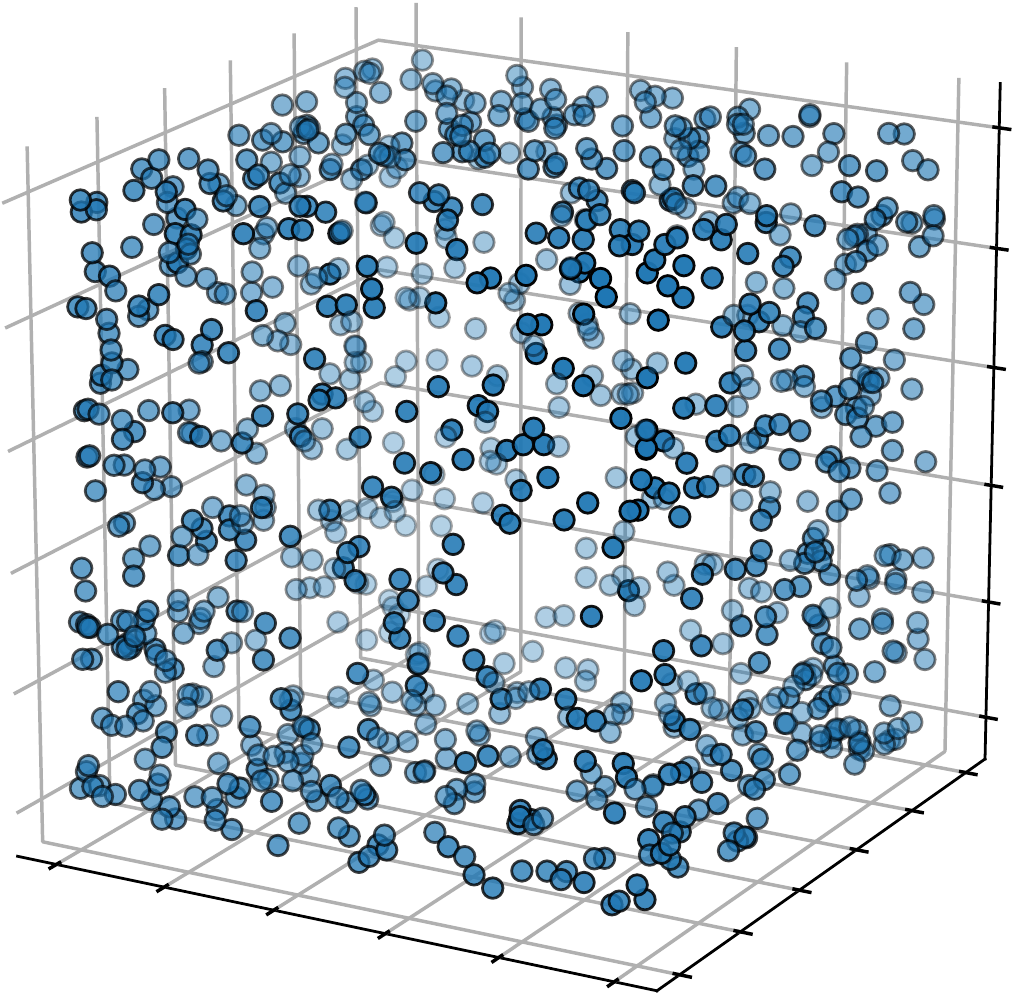}
		\hfill
		\caption{Hypercube (Hc)} 
	\end{subfigure} %
	\hfill
	\begin{subfigure}[b]{\depwa\linewidth}
		\hfill
		\includegraphics[width=0.32\linewidth]{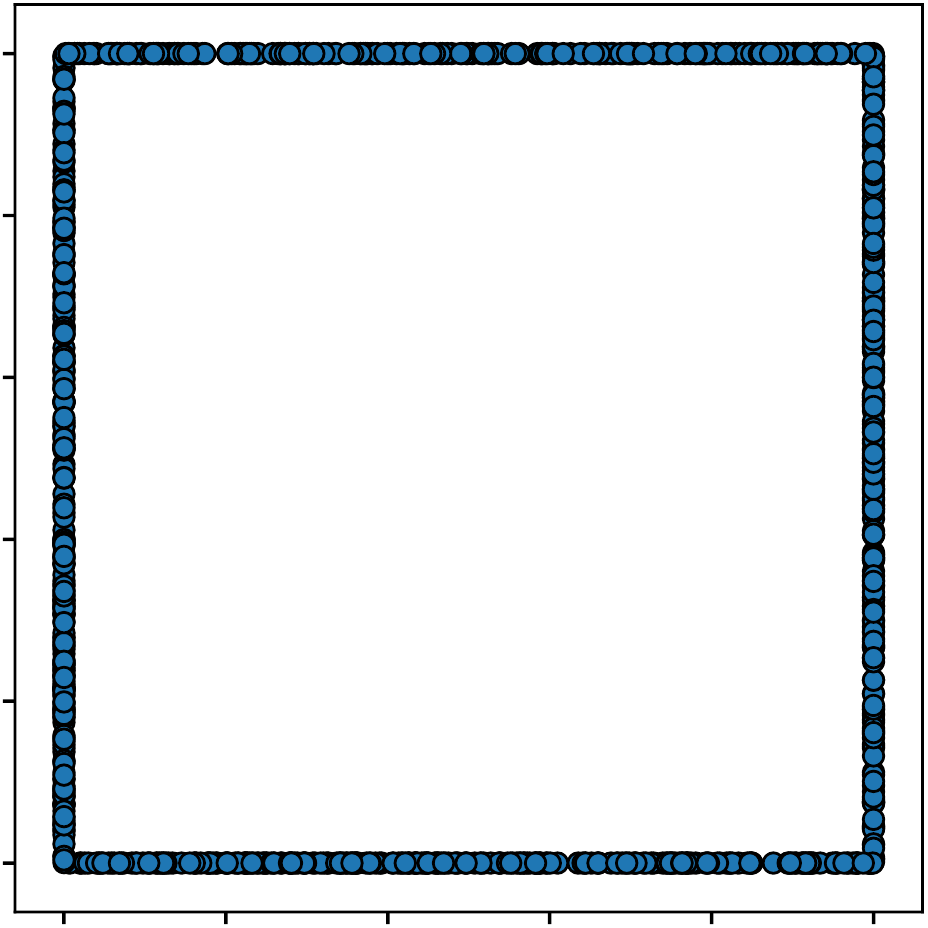} 
		\hfill
		\includegraphics[width=0.35\linewidth]{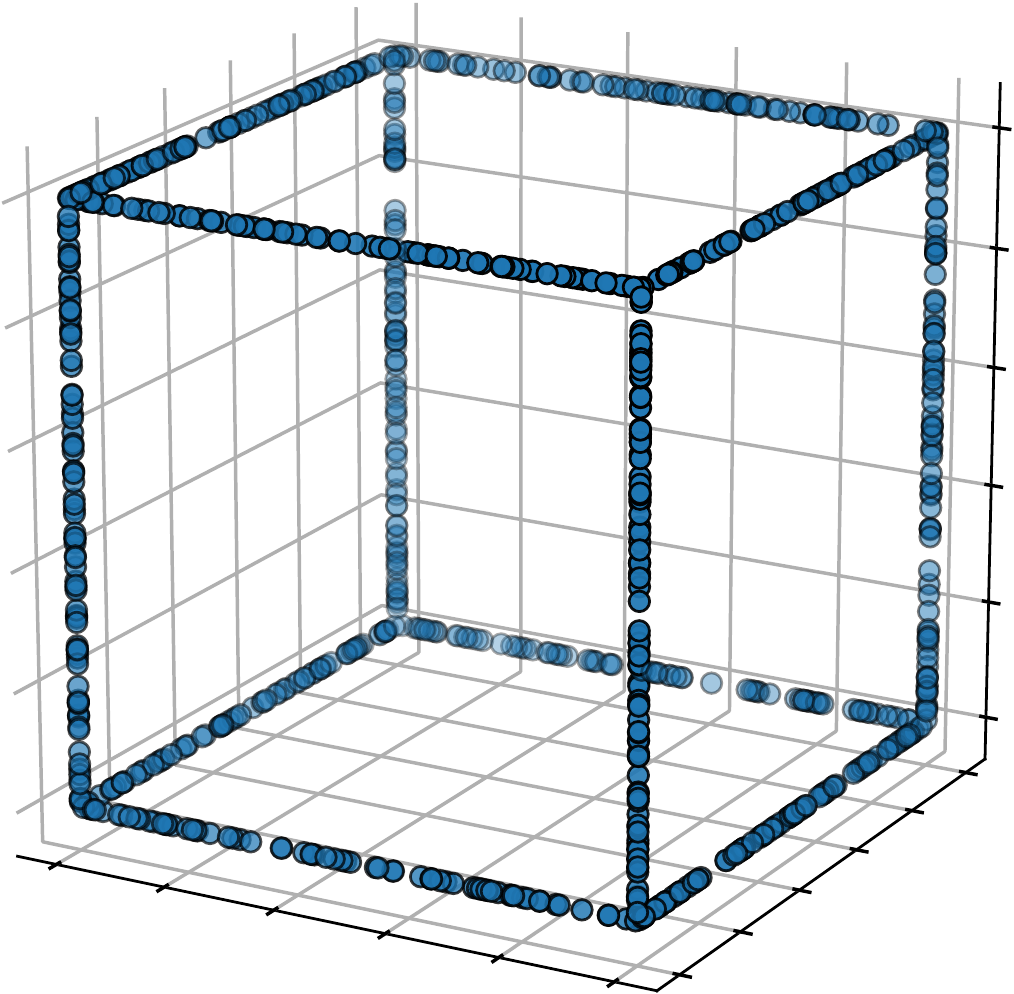} 
		\hfill
		\caption{Hc Graph (HcG)}
	\end{subfigure}
	\hfill
	\begin{subfigure}[b]{\depwa\linewidth}
		\hfill
		\includegraphics[width=0.32\linewidth]{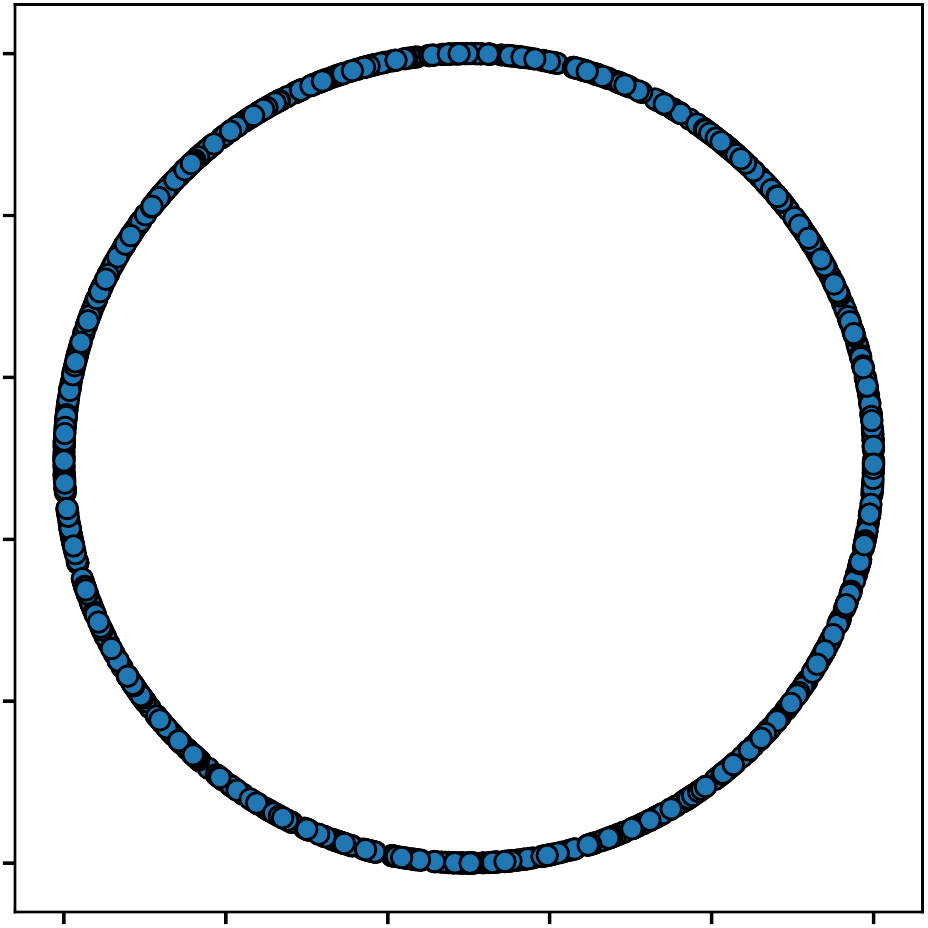} 
		\hfill
		\includegraphics[width=0.35\linewidth]{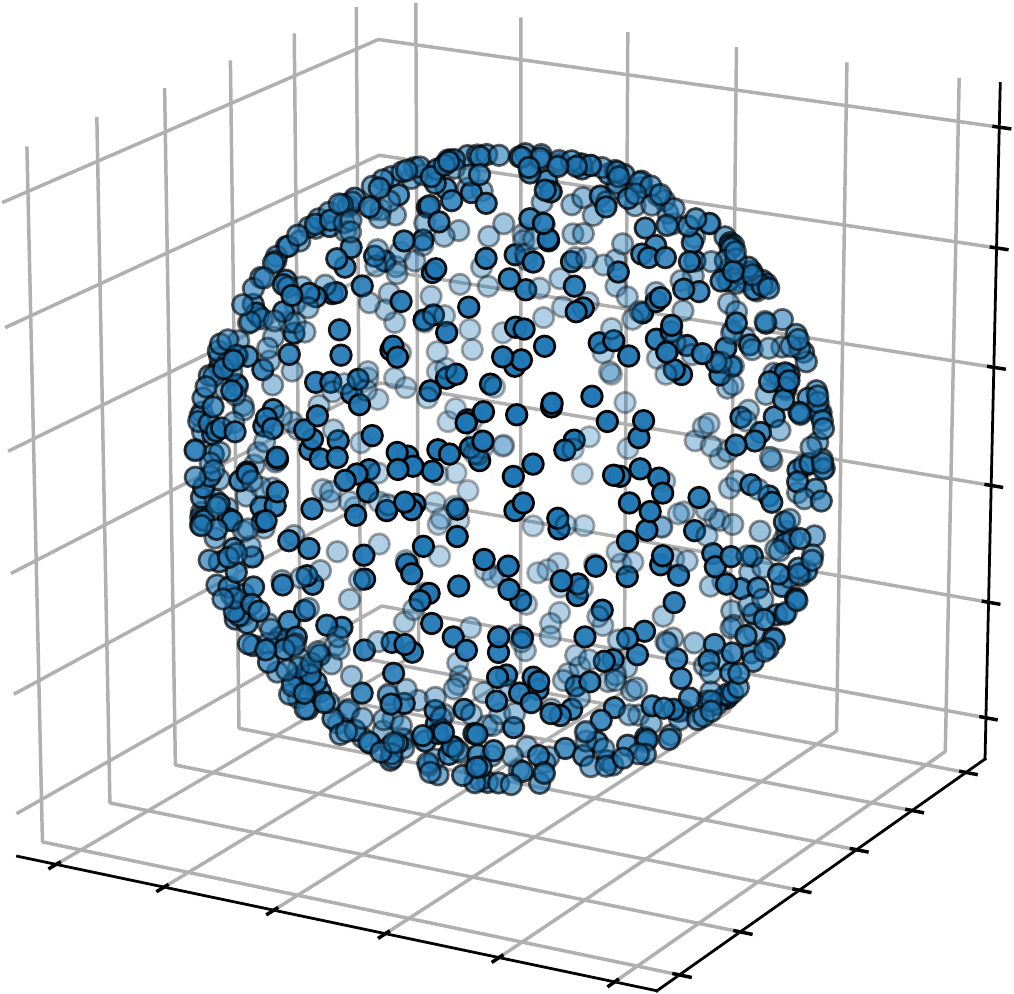} 
		\hfill
		\caption{Hypersphere (Hs)}
	\end{subfigure}
	
	~\\
	\begin{subfigure}[b]{\depwa\linewidth}
		\hfill
		\includegraphics[width=0.32\linewidth]{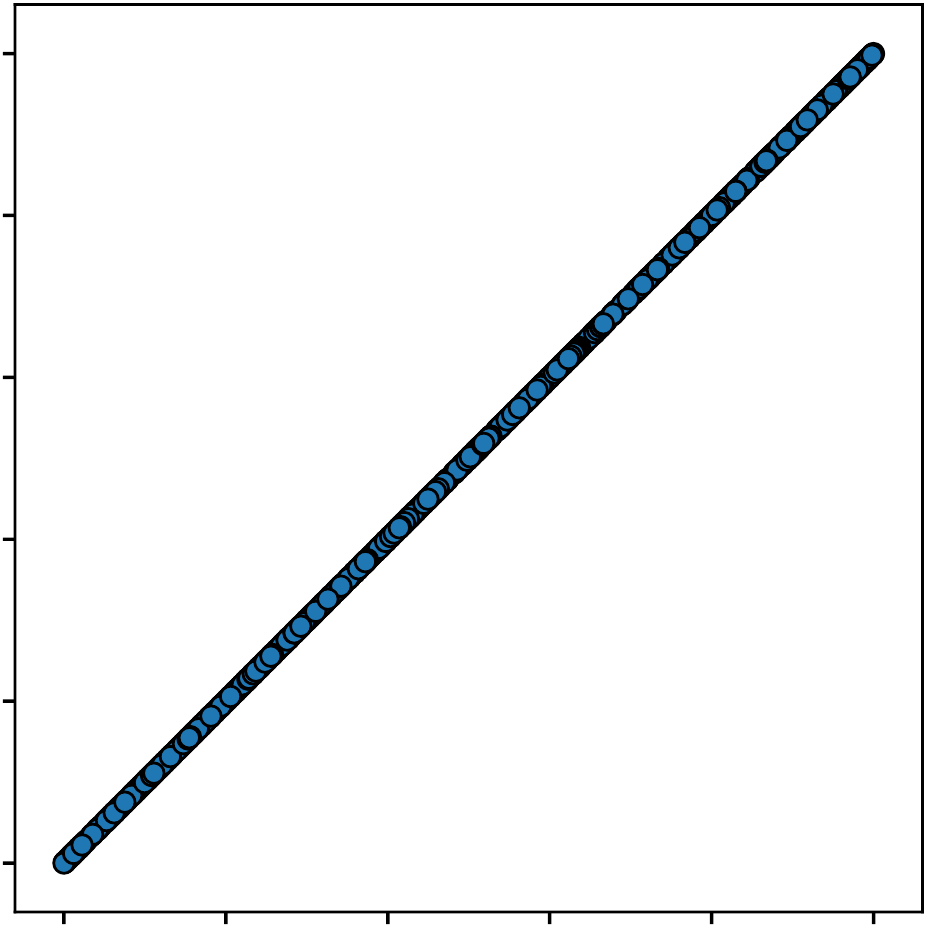} 
		\hfill
		\includegraphics[width=0.35\linewidth]{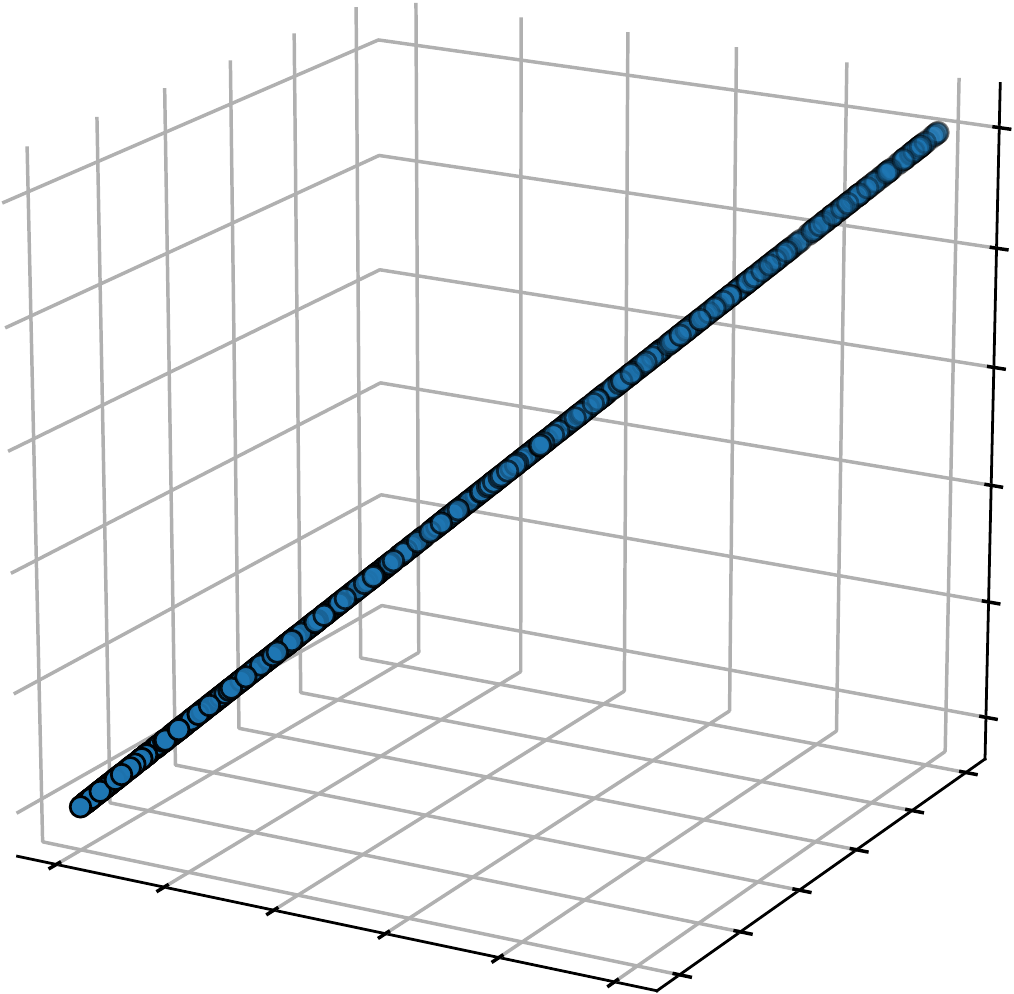}
		\hfill
		\caption{Linear (L)} 
	\end{subfigure} %
	\hfill
	\begin{subfigure}[b]{\depwa\linewidth}
		\hfill
		\includegraphics[width=0.32\linewidth]{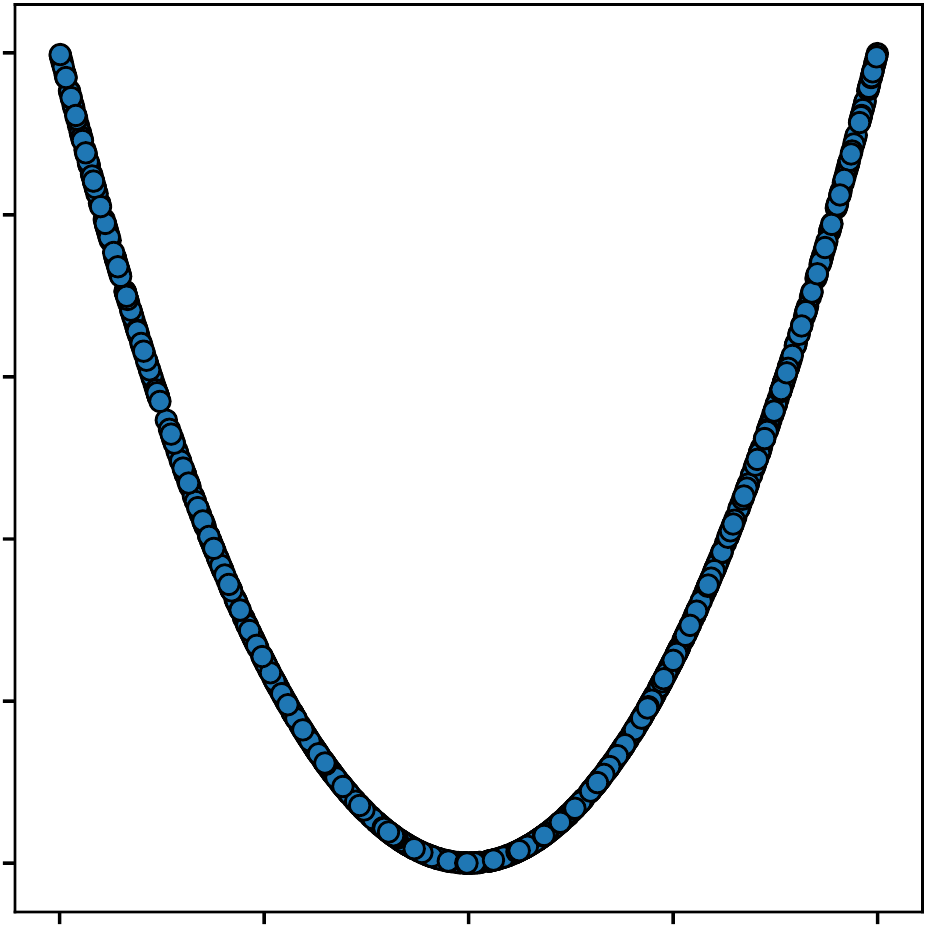}
		\hfill
		\includegraphics[width=0.35\linewidth]{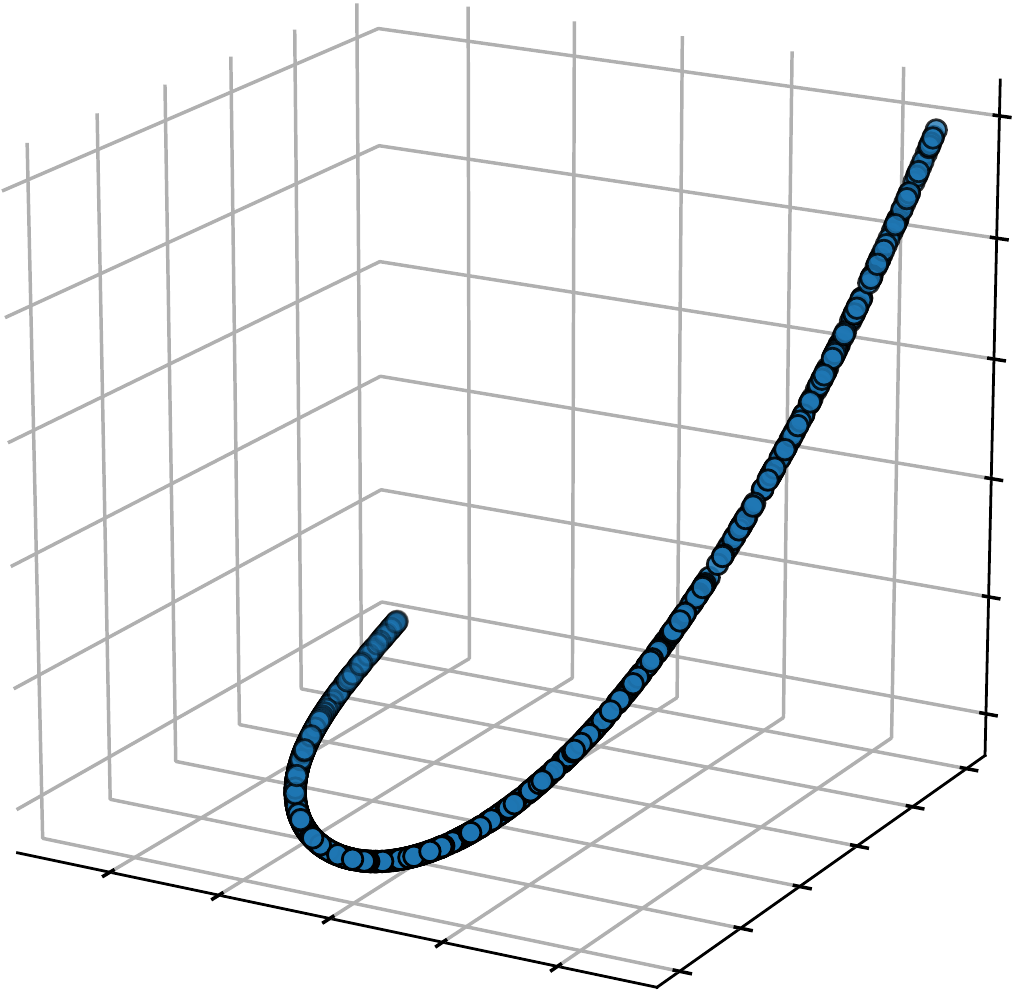}
		\hfill
		\caption{Parabolic (P)} 
	\end{subfigure} %
	\hfill
	\begin{subfigure}[b]{\depwa\linewidth}
		\hfill
		\includegraphics[width=0.32\linewidth]{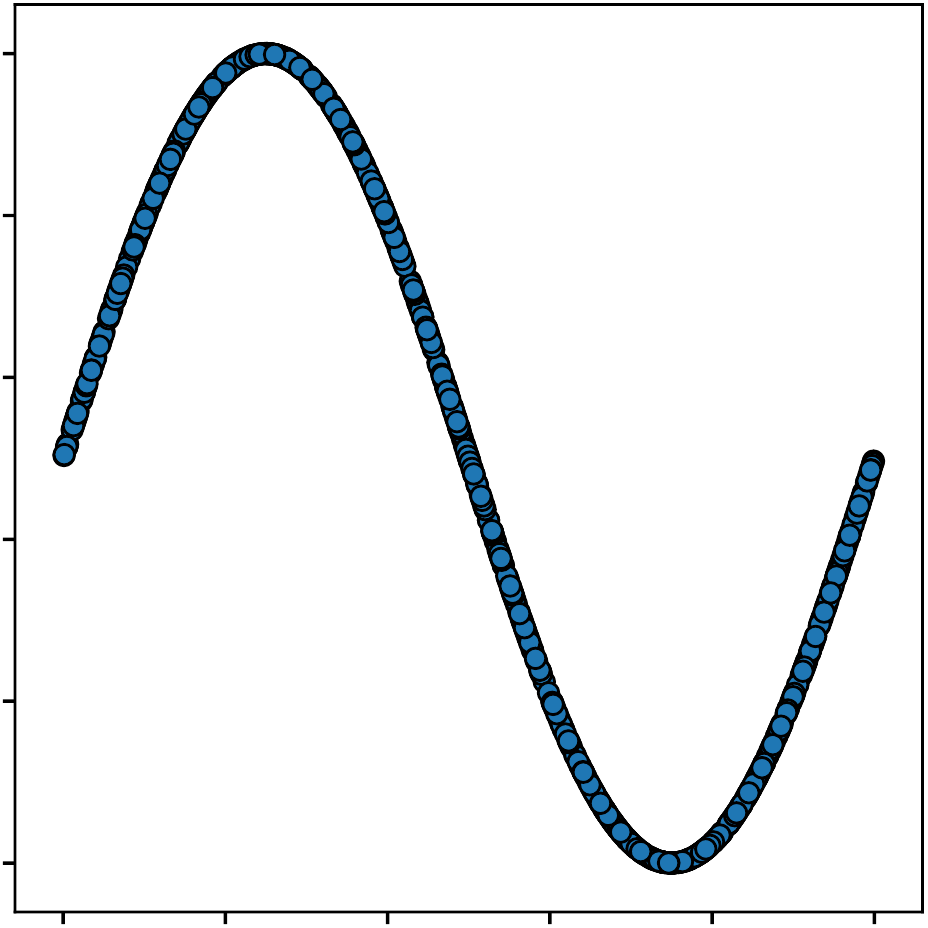} 
		\hfill
		\includegraphics[width=0.35\linewidth]{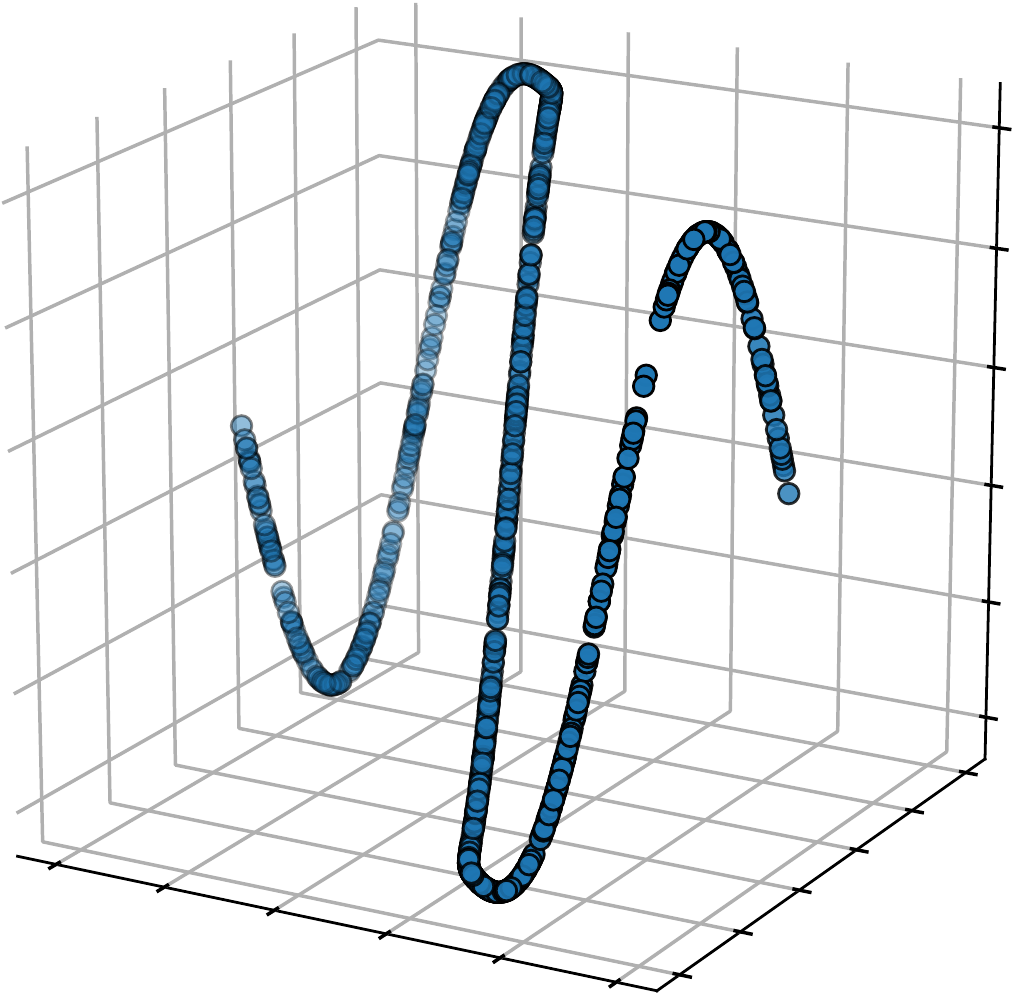}
		\hfill
		\caption{Sine (P=1) (S1)} 
	\end{subfigure} %
	
	~\\
	\begin{subfigure}[b]{\depwa\linewidth}
		\hfill
		\includegraphics[width=0.32\linewidth]{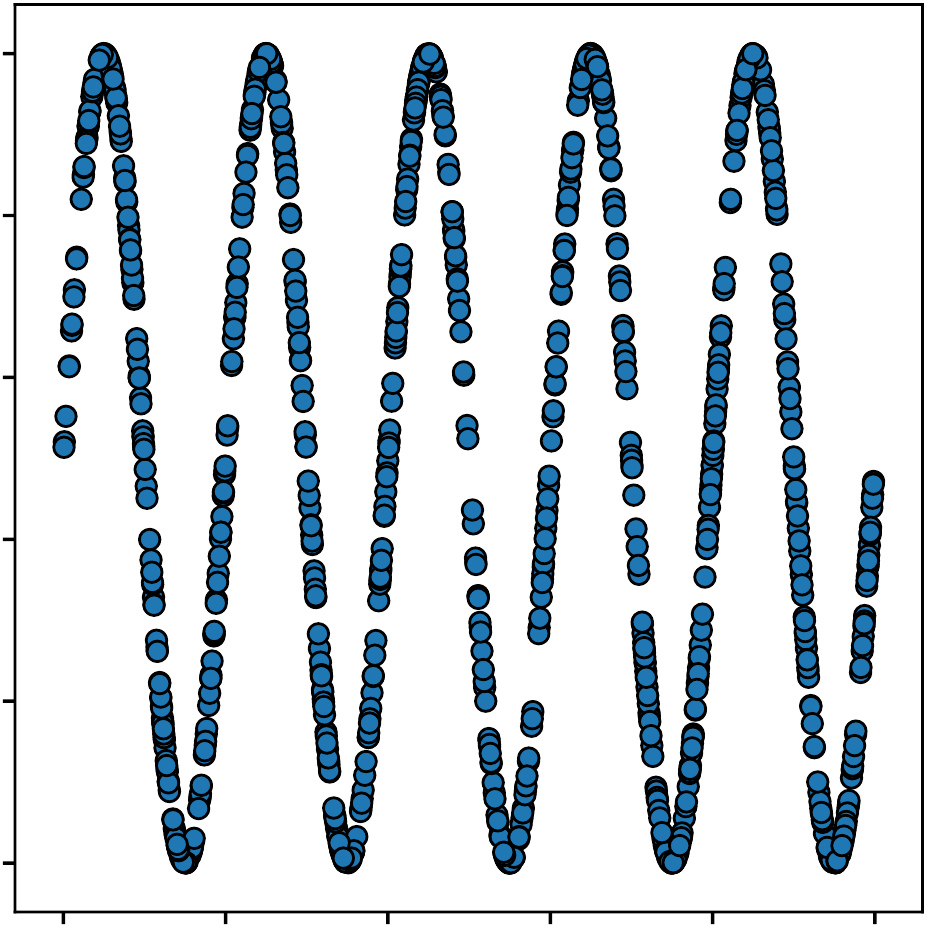} 
		\hfill
		\includegraphics[width=0.35\linewidth]{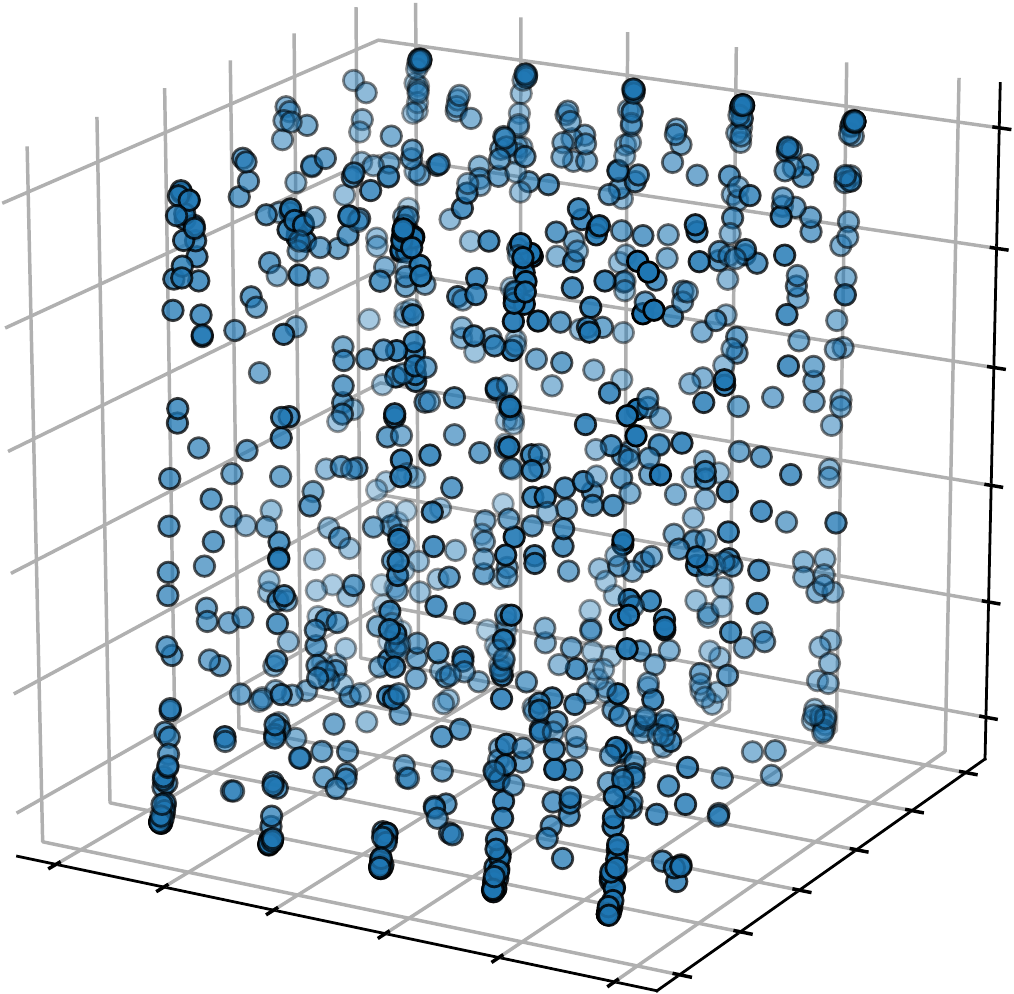}
		\hfill
		\caption{Sine (P=5) (S5)} 
	\end{subfigure} %
	\hfill
	\begin{subfigure}[b]{\depwa\linewidth}
		\hfill
		\includegraphics[width=0.32\linewidth]{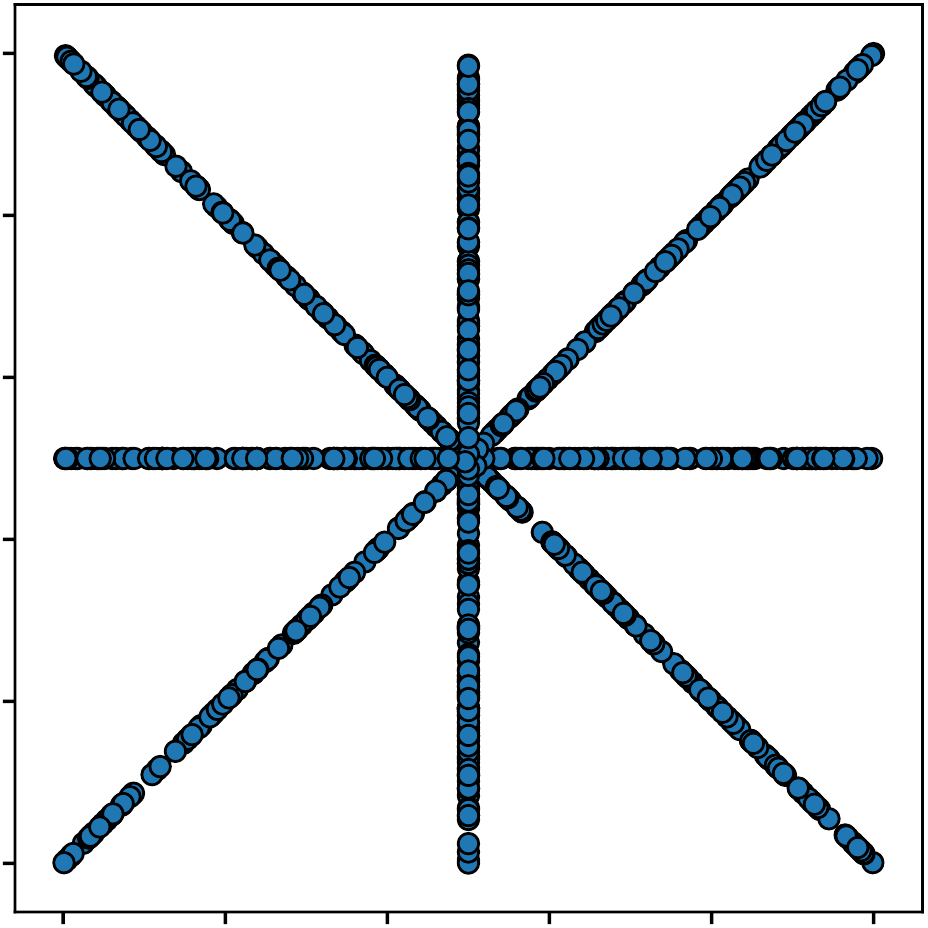} 
		\hfill
		\includegraphics[width=0.35\linewidth]{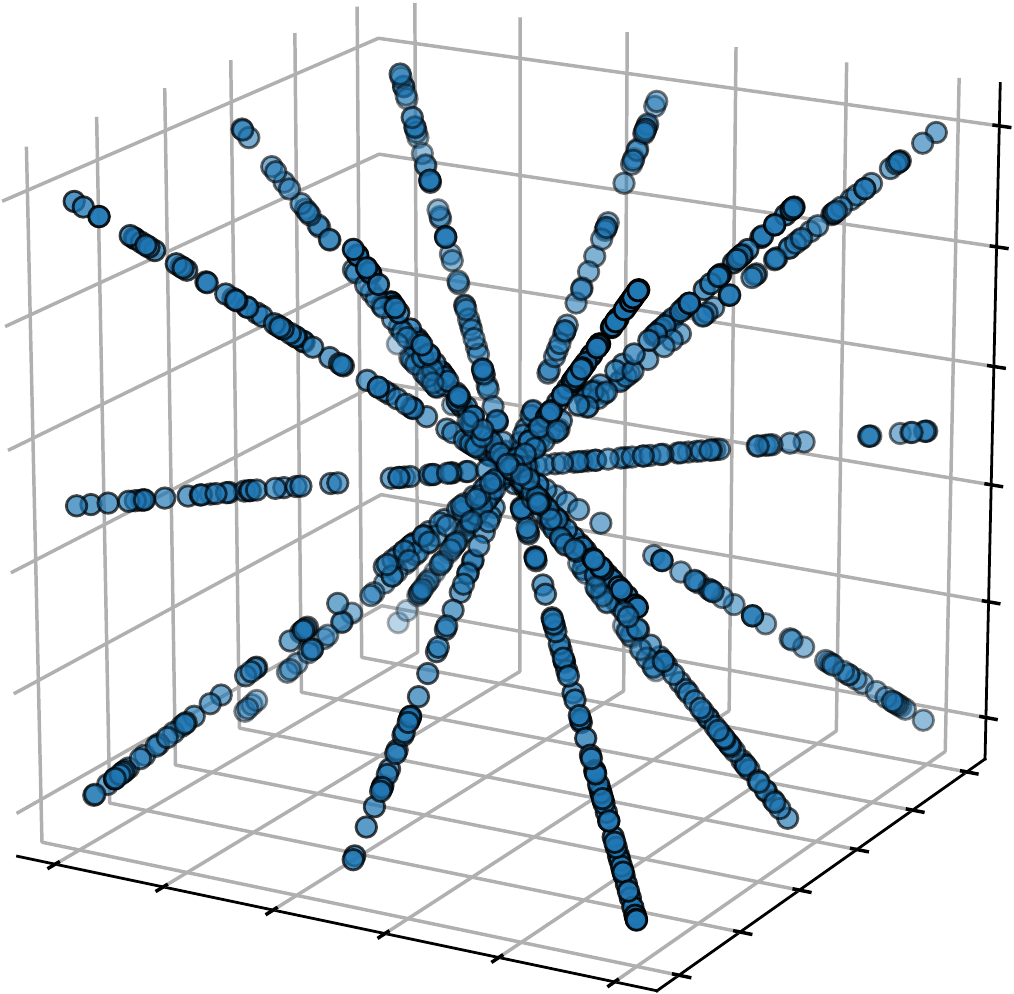}
		\hfill
		\caption{Star (St)} 
	\end{subfigure} %
	\hfill
	\begin{subfigure}[b]{\depwa\linewidth}
		\hfill
		\includegraphics[width=0.32\linewidth]{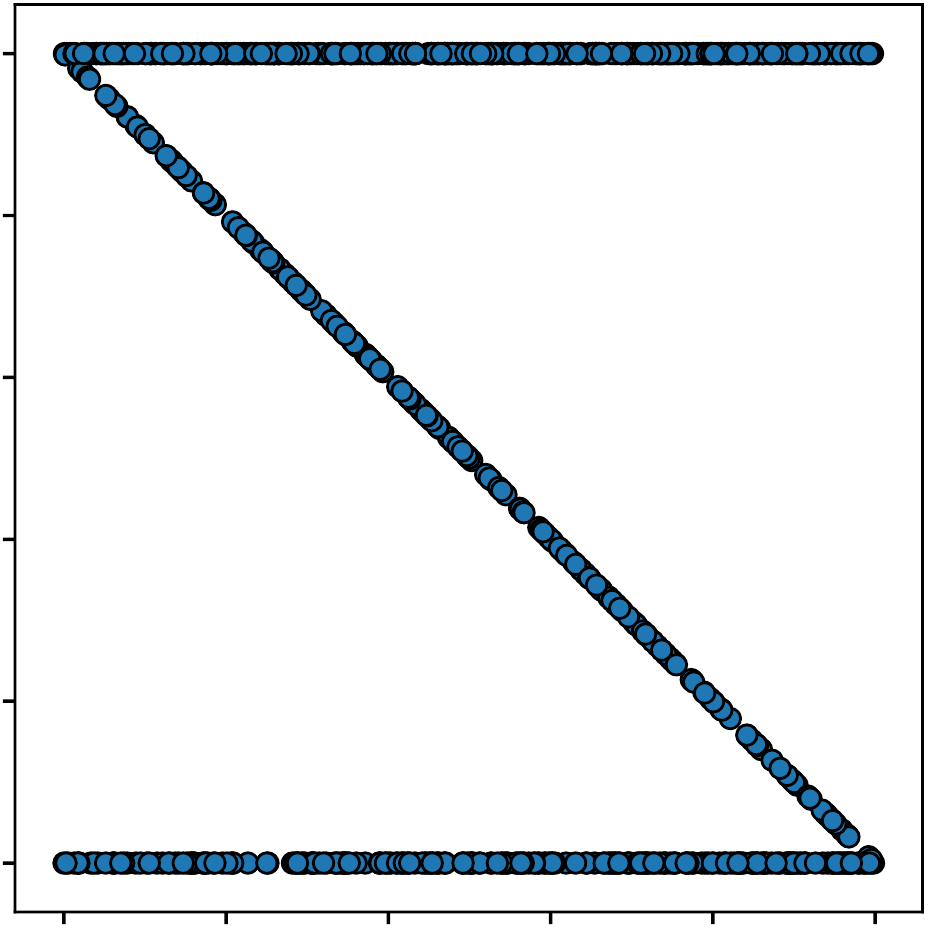} 
		\hfill
		\includegraphics[width=0.35\linewidth]{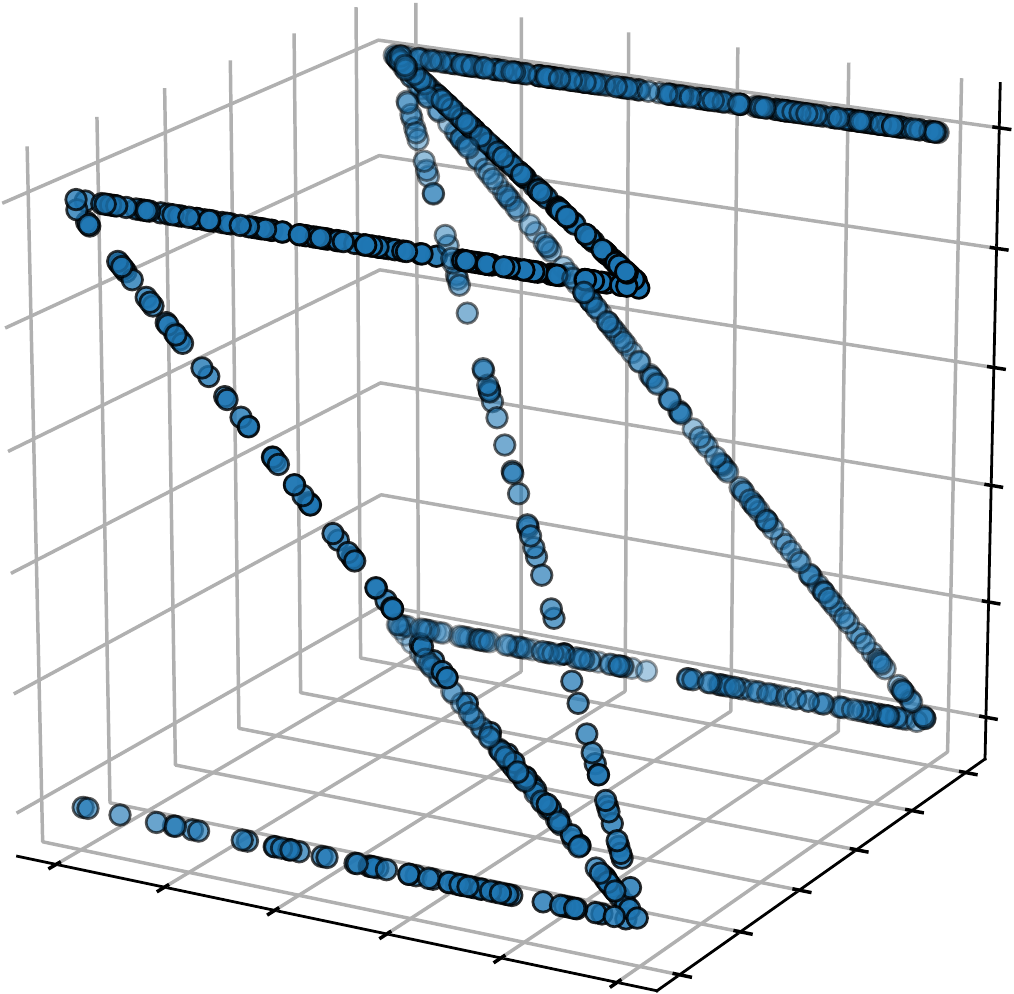} 
		\hfill
		\caption{Z inversed (Zi)}
	\end{subfigure}
	\caption{12 selected multidimensional dependencies}
	\label{fig:deps-plot}
\end{figure}  

\subsubsection{Score Distribution and Statistical Power} A dependency estimator $\mathcal{E}$ is an operator $\mathcal{E}(S) \mapsto \textit{score}$ which computes a \textit{score} for a subspace $S$.
We inspect the \textit{score} of each estimator $\mathcal{E}$ against each dependency \textbf{X}, with increasing noise level $\sigma$. We consider $30$ noise levels, distributed linearly from $0$ to $1$. For better comparability, we also include the independent subspace \textbf{I} in the experiments, where each attribute is i.i.d.\ in $\mathcal{U}[0,1]$.
For each dependency and each noise level, we draw $500$ subspaces to compute the estimate. We record the average (avg) and standard deviation (std) for each estimator and, in analogy to other bivariate and multivariate studies \cite{Nguyen2015UDS, Reshef2011, Kinney2013}, we compute the so-called statistical power. 

\begin{dfn}[Power]
The power of an estimator $\mathcal{E}$ w.r.t.\ \textbf{X} with $\sigma$, $n$ and $d$ is the probability of the \textit{score} of $\mathcal{E}$ to be larger than a $\gamma$-th percentile of the \textit{scores} w.r.t. the independence \textbf{I}:
\normalfont  \begin{align}
\Pr\left ( \mathcal{E}\left(Inst_{n \times d}^{\textbf{X},\sigma}\right) > \left \{\mathcal{E}\left(Inst_{n \times  d}^{\textbf{I},0}\right) \right \}^{P_\gamma} \right ) 
\end{align} 
\end{dfn}

$Inst_{n \times d}^{\textbf{X},\sigma}$ is a random instantiation of a subspace as dependency \textbf{X} with noise level $\sigma$, which has $n$ objects and $d$ dimensions. $\{x\}^{P_\gamma}$ stands for the $\gamma$-th percentile of the set $\{x\}$, i.e., a value $v$ such that $\gamma \%$ of the values in $\{x\}$ are smaller than $v$.
Note that, since the attributes of \textbf{I} are independent, adding noise does not have any effect on dependence, so we set noise to 0 when instantiating \textbf{I}. 
To estimate the power, we draw two sets of $500$ estimates from $\textbf{X}$, $\sigma$ and $\textbf{I}$ respectively: 
\begin{align*}
\Sigma^\mathcal{E}_{\textbf{X}, \sigma} : \left \{ \mathcal{E}\left(Inst_{n \times d}^{\textbf{X},\sigma}\right) \right \}_{i=1}^{500} ~ &&\Sigma^{\mathcal{E}}_\textbf{I} : \left \{ \mathcal{E}\left(Inst_{n \times d}^{\textbf{I},0}\right) \right \}_{i=1}^{500}
\end{align*}
Then, we count the elements in $\Sigma^\mathcal{E}_{\textbf{X}, \sigma}$ greater than $\left \{\Sigma^{\mathcal{E}}_\textbf{I} \right \}^{P_\gamma}$: 
\begin{align}
power^{\textbf{X},\sigma}_{n \times d,\gamma }(\mathcal{E})  = \frac{\left|\left\{x:x \in \Sigma^\mathcal{E}_{\textbf{X}, \sigma} ~\wedge~  x > \left \{\Sigma^{\mathcal{E}}_\textbf{I} \right \}^{P_\gamma}\right \}\right|}{500}
\end{align}

One can interpret $power$ as the probability to correctly reject the independence hypothesis with $\gamma \%$ confidence. 
In other words, the power quantifies how well a dependency measure, such as \textit{MWP}, can differentiate between the independence \textbf{I} and a given dependency \textbf{X} with noise level $\sigma$. 
For our experiments, we choose $\gamma =95$.
In the case of Interaction Information (\textit{II}), values can be negative or positive, depending on whether the dependency is a `synergy' or a `redundancy'. For \textit{II}, we measure power using the absolute value of its score. 

\subsection{General characteristics of \textit{MWP}}
\label{mwp_under_scrutinity}

First, we look at the evolution of the scores of \textit{MWP} regarding the dimensionality $d$, sample size $n$ and $M$. 

\subsubsection{Influence of dimensionality $d$}

\begin{figure}
	\includegraphics[width=1.02\linewidth]{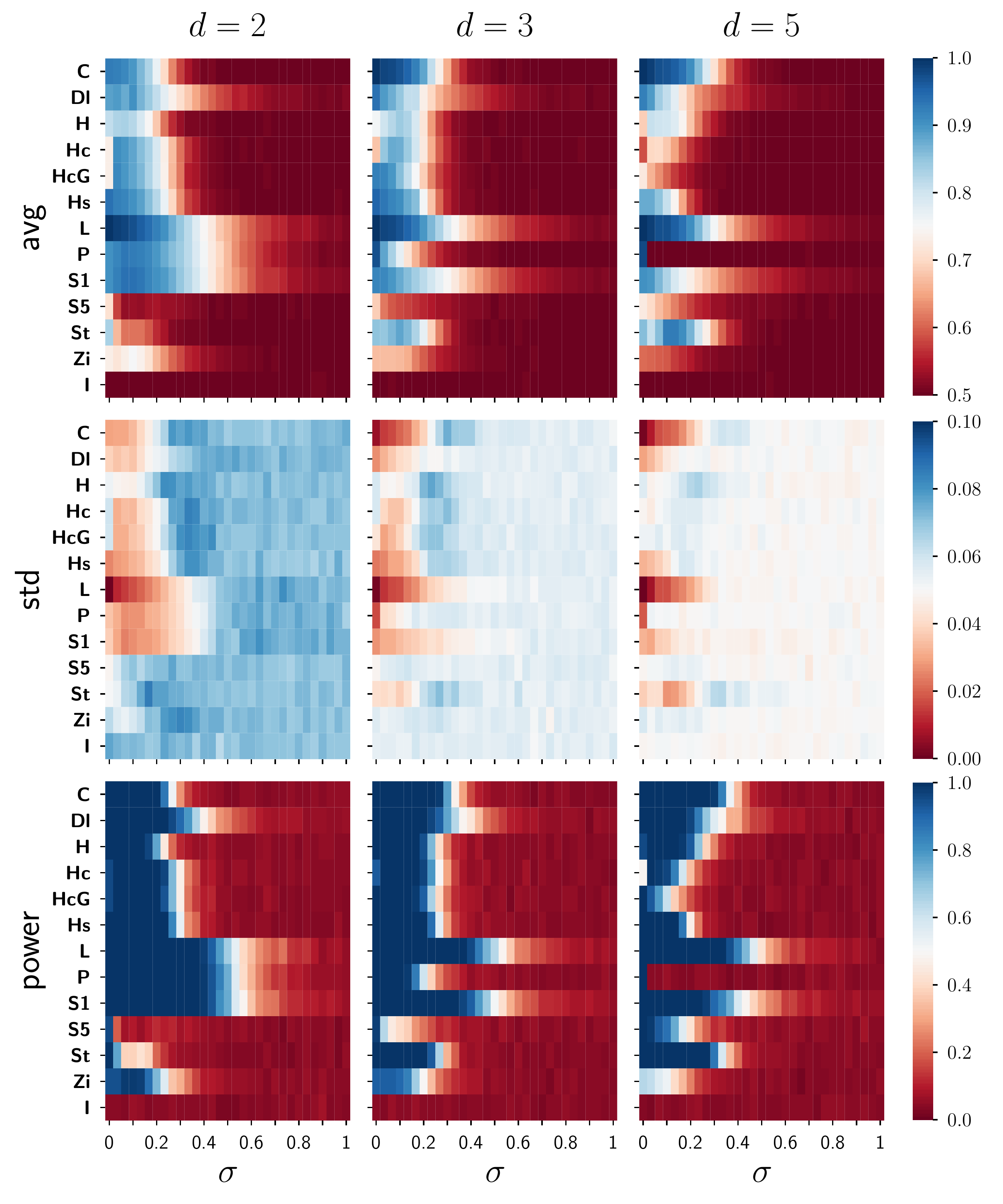}
	\caption{\textit{MWP} w.r.t. dimensionality $d$}
	\label{fig:momentsMWP}
\end{figure} 

Figure~\ref{fig:momentsMWP} graphs the evolution of \textit{MWP} for $d = 2,3,5$. Please note that the figures are best seen in colour. The expectation is that the scores are high for noiseless dependencies, i.e., the left side of the plot is blue, and decrease gradually as we add noise. A noise level $\sigma = 1$ is comparably high, since the data is scaled to $[0,1]$. Thus, the right side of the plot should be red, standing for low scores. 
As we see, the average \textit{MWP} decreases gradually for each dependency. The same level of noise does not seem to affect each estimate equally, also regarding dimensionality. 
For instance, the estimates of \textbf{L}, \textbf{P} and \textbf{S1} are larger at $d=2$.
While the estimates of \textbf{Hc}, \textbf{HcG}, \textbf{P} and \textbf{Zi} decrease with increasing $d$, they increase for \textbf{C} and \textbf{St}. 

The standard deviation of \textit{MWP} increases with noise and decreases with $d$. In particular, \textbf{L}, \textbf{C} and \textbf{Hs} have a low standard deviation. This means that fewer iterations are in fact required to estimate stronger dependencies at a given accuracy. 

The statistical power does not seem to vary much with dimensionality for most dependencies. It decreases with $d$ for \textbf{Hc}, \textbf{HcG}, \textbf{Hs}, \textbf{P} and \textbf{Zi}, while it increases for \textbf{C}, \textbf{S5} and \textbf{St}. 

All in all, each dependency yields a score larger than the independence \textbf{I} up to a certain level of noise, leading to a high power. This indicates that \textit{MWP} is \textbf{general-purpose} (\textbf{R3}). 

\begin{figure}
	\includegraphics[width=\linewidth]{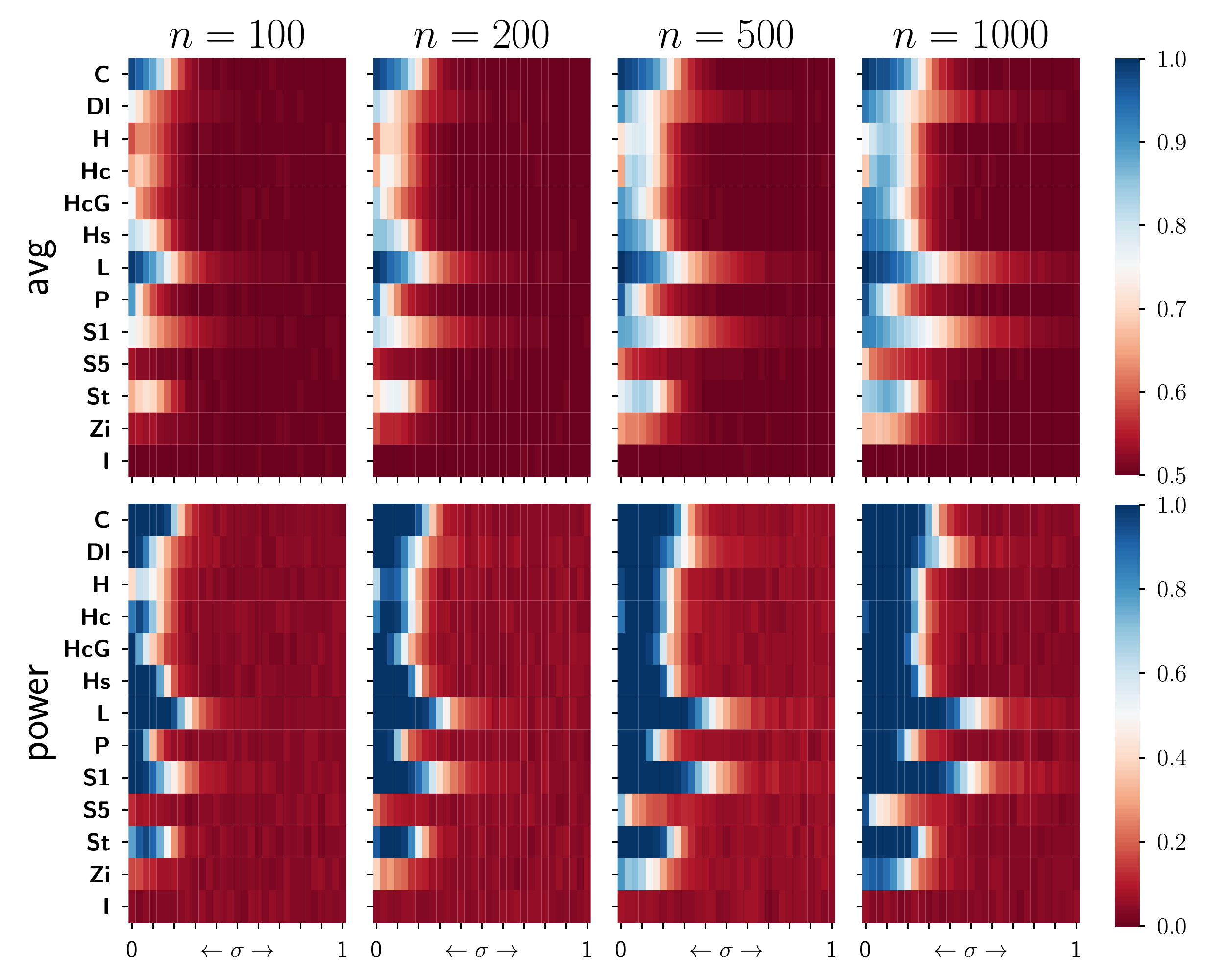} ~
	\includegraphics[width=\linewidth]{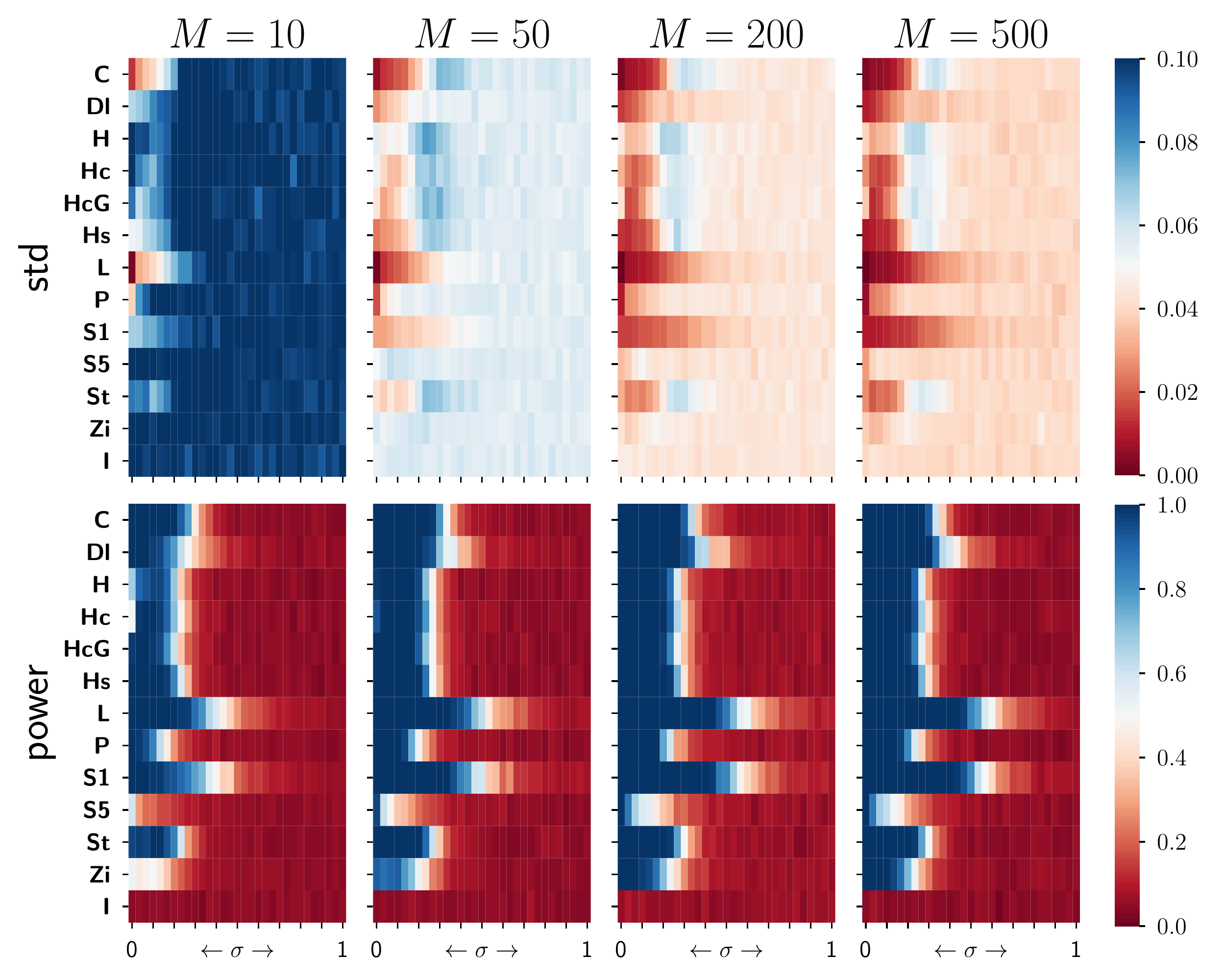}  
	\caption{Power of \textit{MWP} w.r.t. $n$ and $M$}
	\label{fig:NM_MWP}
\end{figure}

\subsubsection{Influence of sample size $n$ and parameter $M$}

Figure~\ref{fig:NM_MWP} shows that power globally increases with $n$, but it is still high for most dependencies with low $n$, provided noise is moderate. 
As we can see, the average score of \textit{MWP} tends to increase with $n$, which explains the gain in power. In fact, that is because \textit{MWP} is \textbf{sensitive} (\textbf{R7}), as we discuss in \mbox{Section \ref{sensitivity}}.
Similarly, power increases slightly as $M$ increases, but the effect is visible only for \textbf{S5} and \textbf{Zi}. This increase of power is easily explained by the fact that the standard deviation of \textit{MWP} decreases, which is what Theorem \ref{"th:hoeffding-chernoff-contrast"} predicted: with more iterations, the values concentrate more around $\mathcal{C}$. 

In the end, we see that \textit{MWP} is already useful for small $n$ or small $M$, even though more iterations or more data samples yield higher power when data is noisy. 

\subsection{Score Distribution and Statistical Power}
\label{distribution}

\begin{figure*}
	\includegraphics[width=\linewidth]{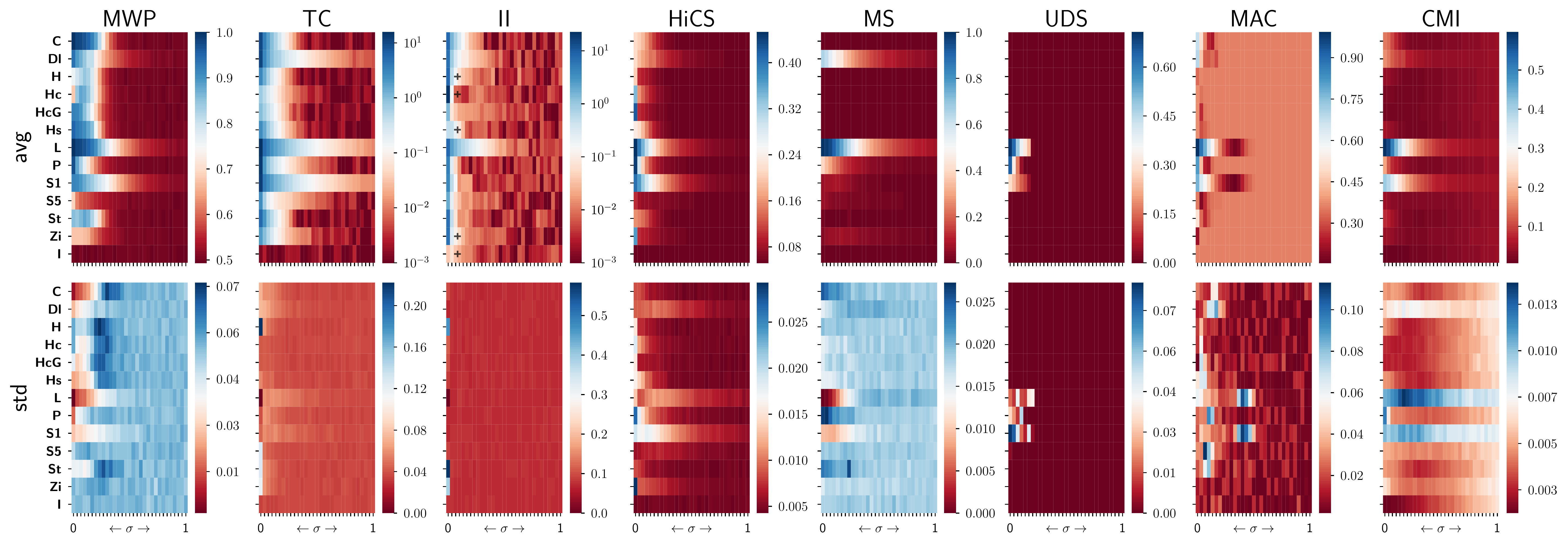} 
	\caption{Distribution of dependency estimation scores, $d=3$}
	\label{fig:contrast_all}
\end{figure*}

We now observe the distribution of the scores for each approach in Figure \ref{fig:contrast_all}. First, we see that the average score of \textit{MWP} is most similar to \textit{TC}. \textit{TC} however is unbounded, and a logarithmic scale is necessary to visualize it. This means that the estimates of TC change very abruptly. 

\textit{II} can yield positive or negative values. Since both cases are interesting, we visualize the absolute value of \textit{II} with a logarithmic scale. We mark the dependencies which obtain a positive score in their noiseless form with a plus sign. 
Like \textit{TC}, \textit{II} assigns high absolute scores to every noiseless dependency. However, the score decreases rapidly with noise, except for \textbf{L}.

\textit{HiCS} shows a similar behaviour as \textit{MWP}, except that the scores decrease faster, and that a large number of dependencies starts with a relatively low score, even in the noiseless form, such as \textbf{C}, \textbf{Dl}, \textbf{H},  \textbf{Hs}, \textbf{S5} and \textbf{St}. 

Next, \textit{MS} and \textit{UDS} are restricted to monotonous and bijective functional relationships respectively. They can detect only 3 out of 12 dependencies. \textit{MAC} and \textit{CMI} behave curiously. Their scores change noticeably only for \textbf{C}, \textbf{Dl}, \textbf{L}, \textbf{P} and \textbf{S1}.
The values of \textit{MAC} also change very abruptly and even non-monotonously with noise. For example, \textbf{L} and \textbf{S1} obtain lower scores with a noise level of 0.3 than with higher noise levels. \textit{CMI} evolves smoothly. However, for many dependencies, including \textbf{I}, the score increases again with more noise: The shades on the right are lighter, which shows a bias towards noise, independently from the underlying relationship.

By looking at \textit{MWP} and \textit{MS}, we see that the standard deviation behaves similarly: It decreases as the score increases. We observe the opposite for \textit{HiCS}. The standard deviation of \textit{CMI} reaches its highest level at a certain noise level, around $0.2$ for \textbf{L}, and tends to increase slightly again with more noise. For \textit{UDS} and \textit{MAC}, the evolution of the standard deviation looks very unstable. The standard deviation of \textit{TC} and \textit{II} does not change much, except for noiseless dependencies.

While the scores of \textit{HiCS}, \textit{UDS}, \textit{MAC} and \textit{CMI} are expected to be in $[0,1]$, the theoretical maximum or minimum is never reached, even if our benchmark features both strong and weak dependencies. On the other hand, \textit{MWP} and \textit{MS} exploit all the values of their range, being $[0.5,1]$ and $[0,1]$ respectively. Thus, they are easier to interpret (\textbf{R6}).

\begin{figure}
\begin{subfigure}[b]{\linewidth}
	\includegraphics[width=\linewidth]{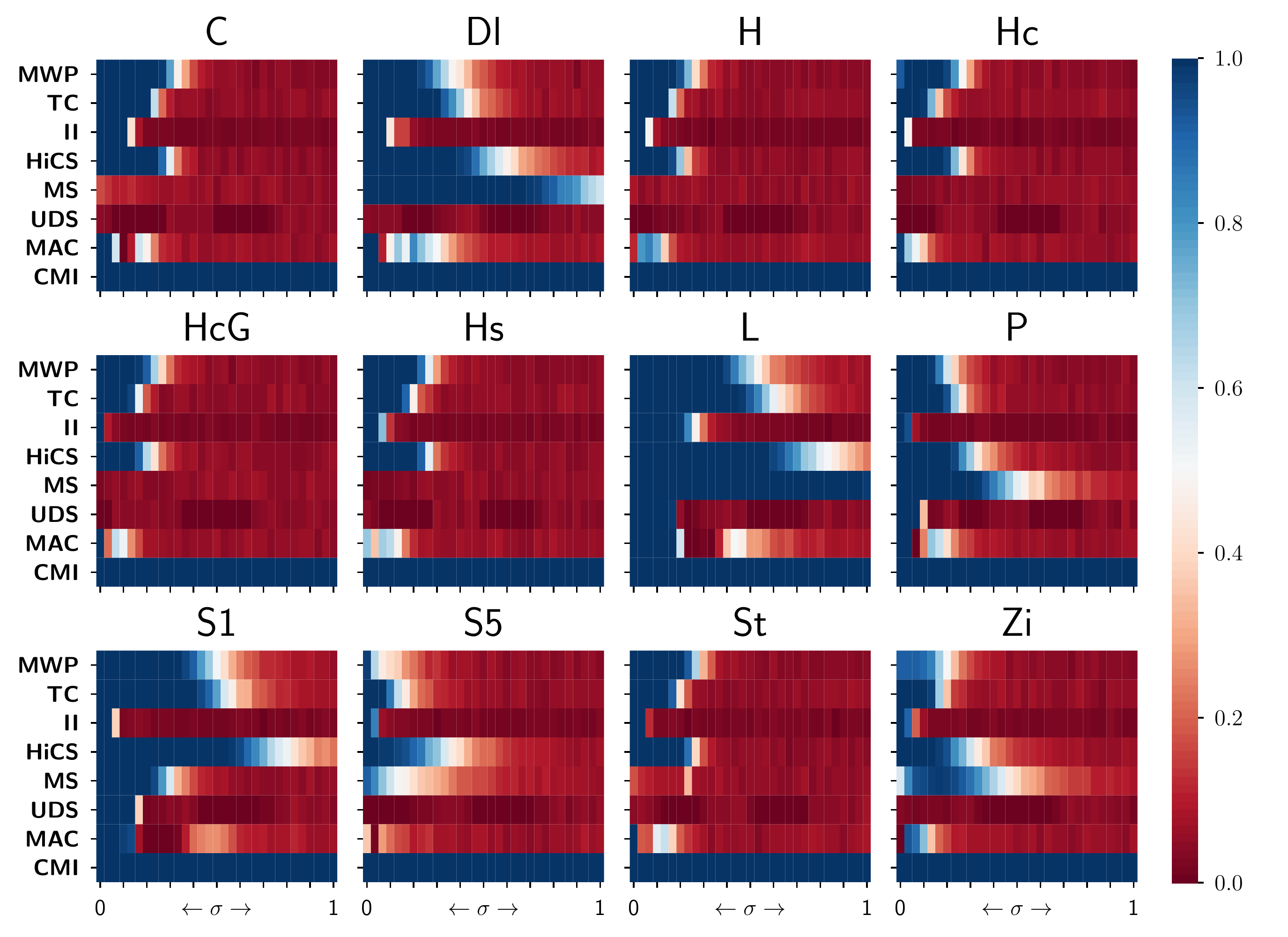} 
	\caption{$d=3$} 
\end{subfigure}
\begin{subfigure}[b]{\linewidth}
	\includegraphics[width=\linewidth]{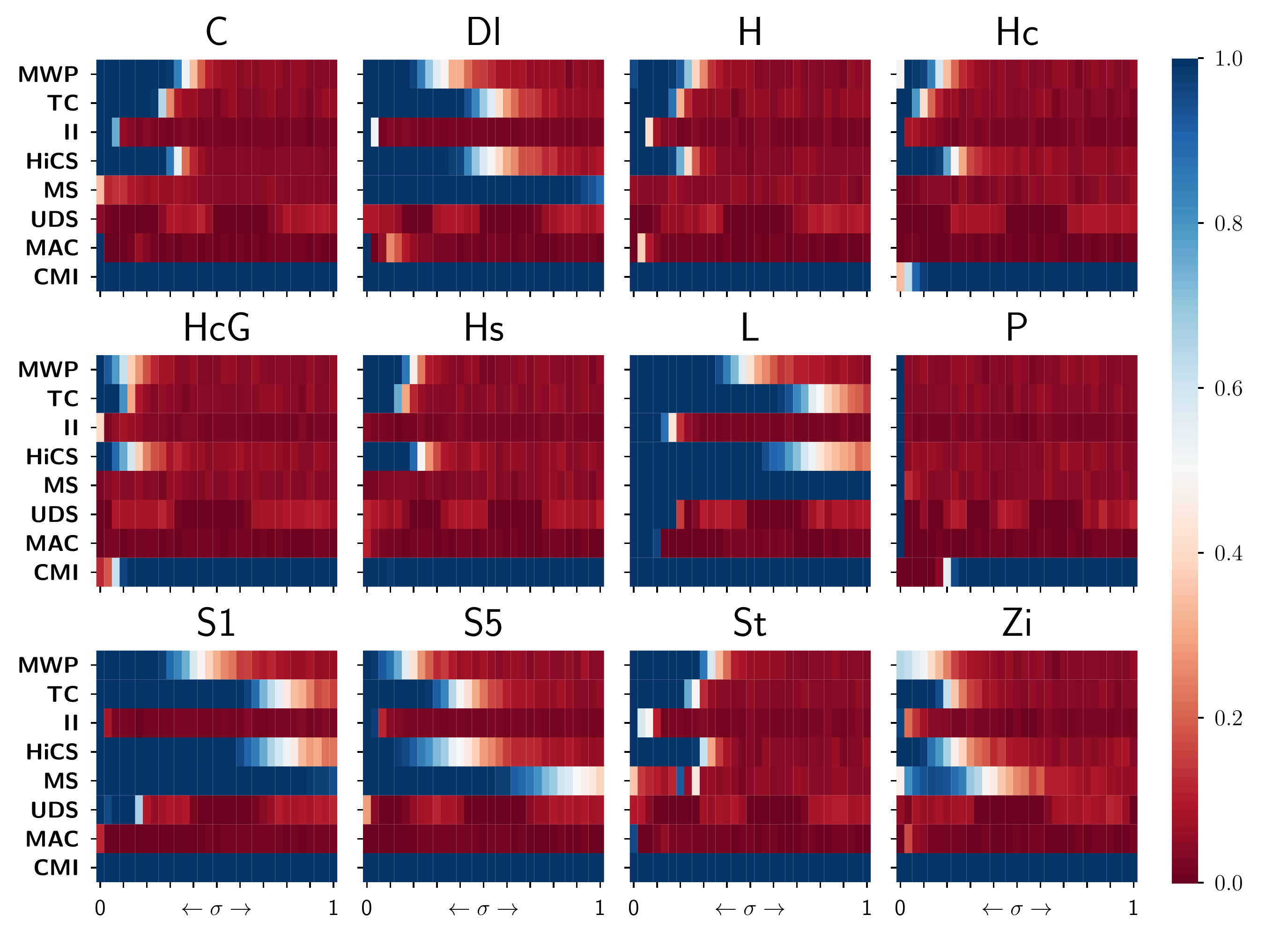} 
	\caption{$d=5$} 
\end{subfigure}
\caption{Power against each dependency}
\label{fig:E_Power}
\end{figure}

Figure \ref{fig:E_Power} reveals that \textit{MWP}, \textit{TC} and \textit{HiCS} achieve high power in any situation up to a certain extent of noise. 
\textit{MWP} shows slightly more power with \textbf{C}, \textbf{H}, \textbf{Hc}, \textbf{HcG}, \textbf{Hs} and \textbf{St}.  
\textit{II} can detect almost every dependency but the power decreases rapidly with noise and dimensionality. \textit{MS} detects \textbf{Dl}, \textbf{L}, \textbf{P}, \textbf{S1}, \textbf{S5} and \textbf{Zi}, but misses all other dependencies. 
\textit{MAC} looks unstable, since its power evolves in a non-monotonous way and decreases with increasing dimensionality by much. In fact, it is not able to detect most dependencies for $d=5$.  
\textit{UDS} can only detect \textbf{L}, \textbf{P} and \textbf{S1}, a clear limitation. 
\textit{CMI} has maximal power for each dependency and noise level for $d=3$, which is unrealistic: \textit{CMI} reaches its lowest score against the noiseless \textbf{I}, our baseline for power. This means that \textit{CMI} cannot distinguish between noise and dependence. 

\subsection{Sensitivity}
\label{sensitivity}
\textbf{R7} states that estimators should also reflect the strength of the observed effect w.r.t. the number of observations. Figure \ref{fig:All_sensitivity} graphs the average score from 500 instances of each dependency with a small noise level of $1/30$. The average of \textit{MWP} obtained for each dependency converges to 1 consistently with more samples, except for the independence. Its values stabilize around $0.5$. This means that \textit{MWP} is \textbf{sensitive} (\textbf{R7}). 

\textit{TC} behaves similarly to \textit{MWP}: When the sample size increases, the score tends to increase as well. However, it is not bounded. While the scores of \textit{II} seem to increase with sample size, they decrease in terms of absolute value, except for \textbf{Hs}. \textit{MS} is completely insensitive to changes in the sample size. \textit{HiCS}, \textit{UDS}, \textit{MAC} and \textit{CMI} behave antagonistically: Their scores tends to go down as the sample size increases, even in the case of the independence \textbf{I}. This implies that their minimum or maximum score varies with the sample size, highlighting also interpretability (\textbf{R6})  problems. 

\begin{figure}
	\includegraphics[width=\linewidth]{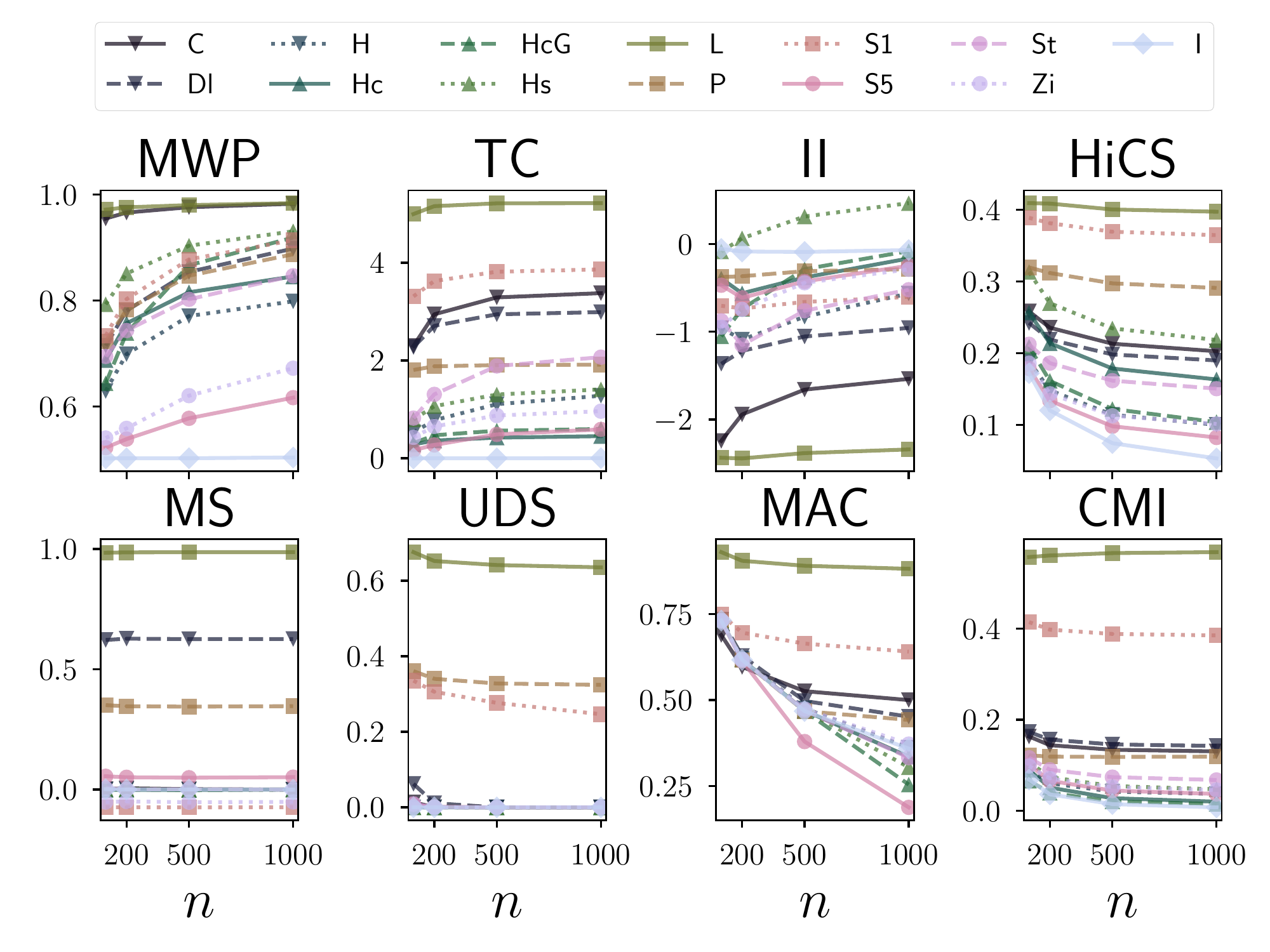} 
	\caption{Average score w.r.t. $n$, $\sigma = 1/30$}
	\label{fig:All_sensitivity}
\end{figure} 

\subsection{Robustness}
\label{robustness}
Data is often imperfect, i.e., values are rounded or trimmed. In some cases, this may lead to wrong estimates, e.g., an independent space is declared as strongly dependent. We simulate data imperfections by discretising a 3-dimensional linear dependency into a number $\omega$ of discrete values from 100 to 1. With only one value, the space is completely redundant, i.e., its \textit{contrast} should be minimal. We compare the power of \textit{MWP} and of the other approaches against the linear dependency \textbf{L} and the independence \textbf{I} for different levels of discretisation.  Figure \ref{fig:E_Discrete_Power} displays the results. Since \textit{TC} and \textit{II} rely on a nearest neighbours algorithm, they fail when the same observation is present more than $k$ times i.e., they are by design not robust; we exclude them from the analysis. 

\textit{HiCS} yields high power in the case of discrete values, even with \textbf{I}. This is because \textit{HiCS} uses the \textit{Kolmogorov-Smirnov} test which assumes continuous data. Thus, \textit{HiCS} is not robust. Also, the power of \textit{CMI} wrongly increases as we add noise to \textbf{I}, provided that the discretisation level $\omega$ is not less than $10$. 
This explains why the power of \textit{CMI} is high for every dependency in Section \ref{distribution}. \textit{CMI} rejects the independence for independent spaces as well, i.e., it is not robust. On the other hand, \textit{MWP}, \textit{MS}, \textit{UDS} and \textit{MAC} appear \textbf{robust} (\textbf{R8}).

In Figure \ref{fig:E_Discrete_average}, we see that the score of \textit{CMI} tends to increase slightly for \textbf{I} as we add noise, provided that $\omega > 5$. Also, the score of \textit{HiCS} increases for both \textbf{I} and \textbf{L} when $\omega \leq 5$. \textit{MAC} converges to $0.4$ as noise increases for $\omega > 10$. On the other hand, \textit{MWP} converges to $0$ as the space becomes discrete. This is an interesting feature of our estimator: discrete spaces are of lower interested, since the notion of \textit{contrast} is not defined there. It allows analysts to draw a line between discrete and real-valued attributes in terms of interestingness.  

\begin{figure}
\begin{subfigure}[b]{\linewidth}
	\includegraphics[width=\linewidth]{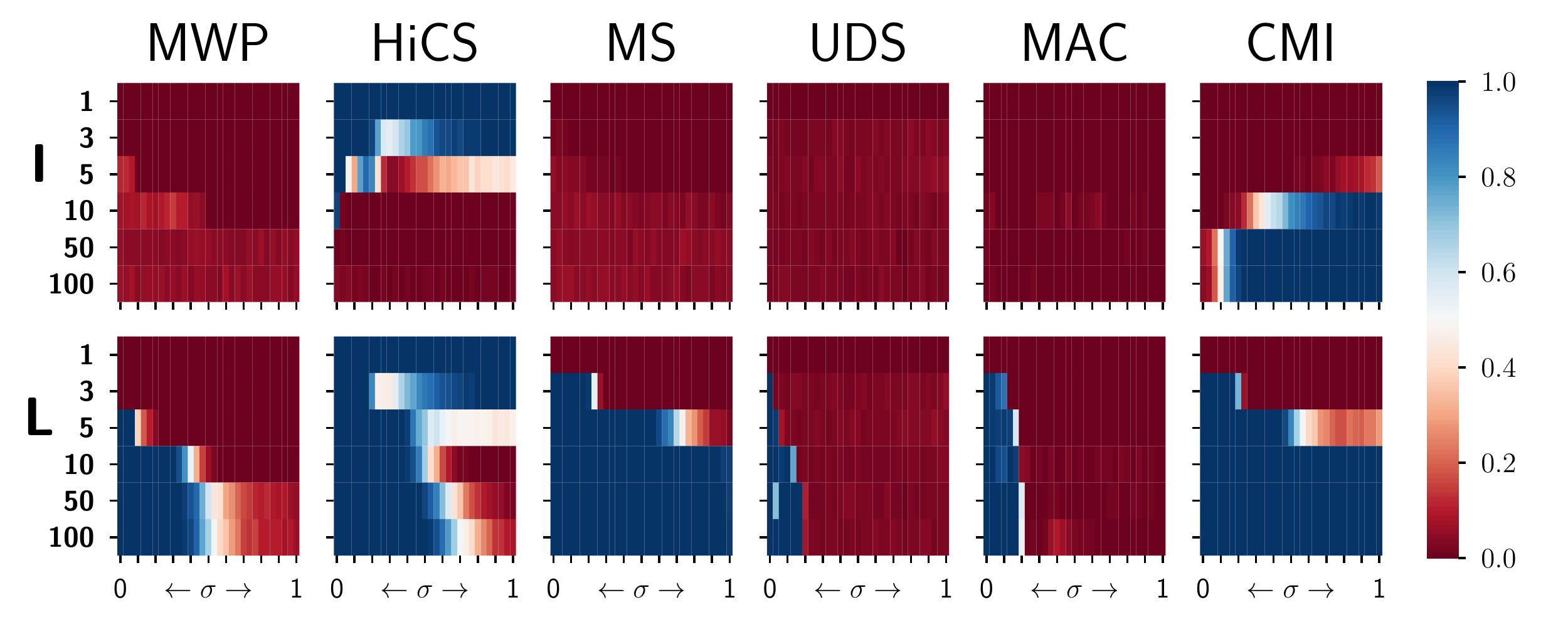} 
	\caption{Power} 
	\label{fig:E_Discrete_Power}
\end{subfigure}
\begin{subfigure}[b]{\linewidth}
	\includegraphics[width=\linewidth]{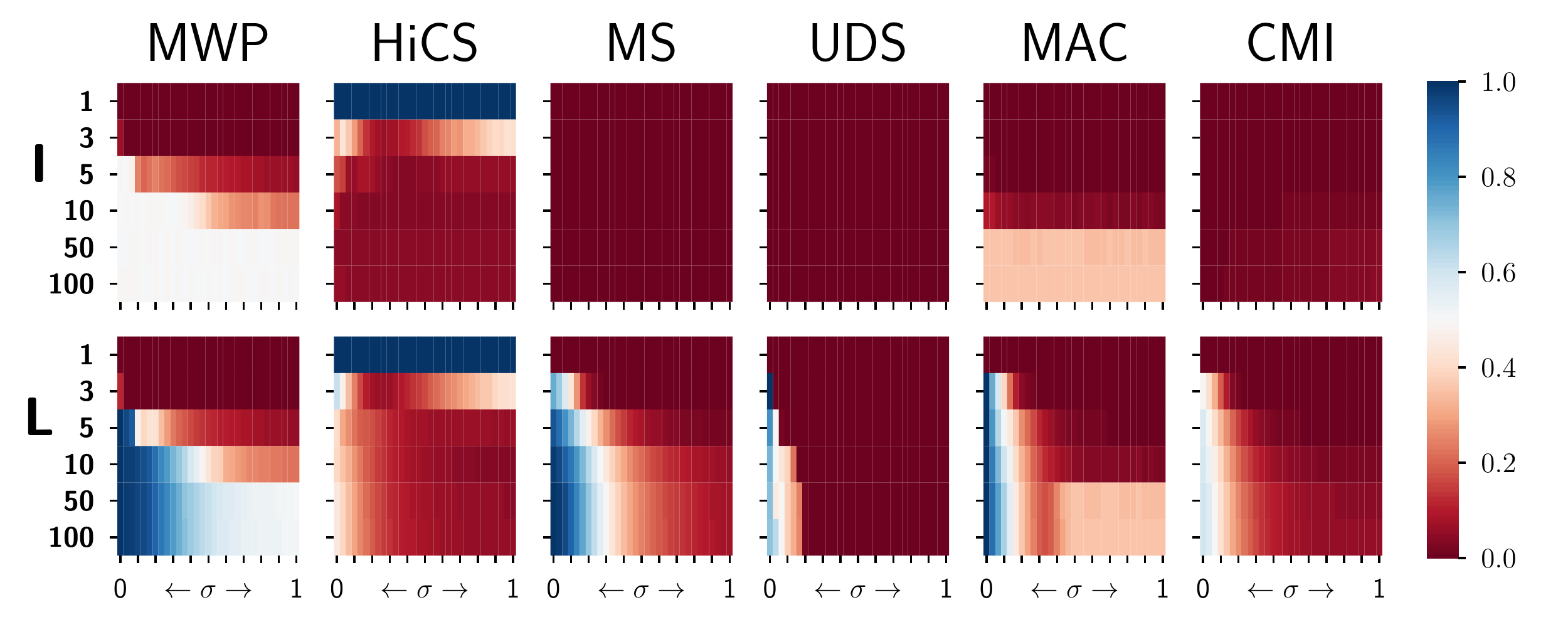} 
	\caption{Average} 
	\label{fig:E_Discrete_average}
\end{subfigure}
\caption{Power and average score of each approach w.r.t. $\omega$}
\label{fig:E_Discrete}
\end{figure}

\subsection{Scalability}
\label{scalability}

\begin{figure}
	\includegraphics[width=\linewidth]{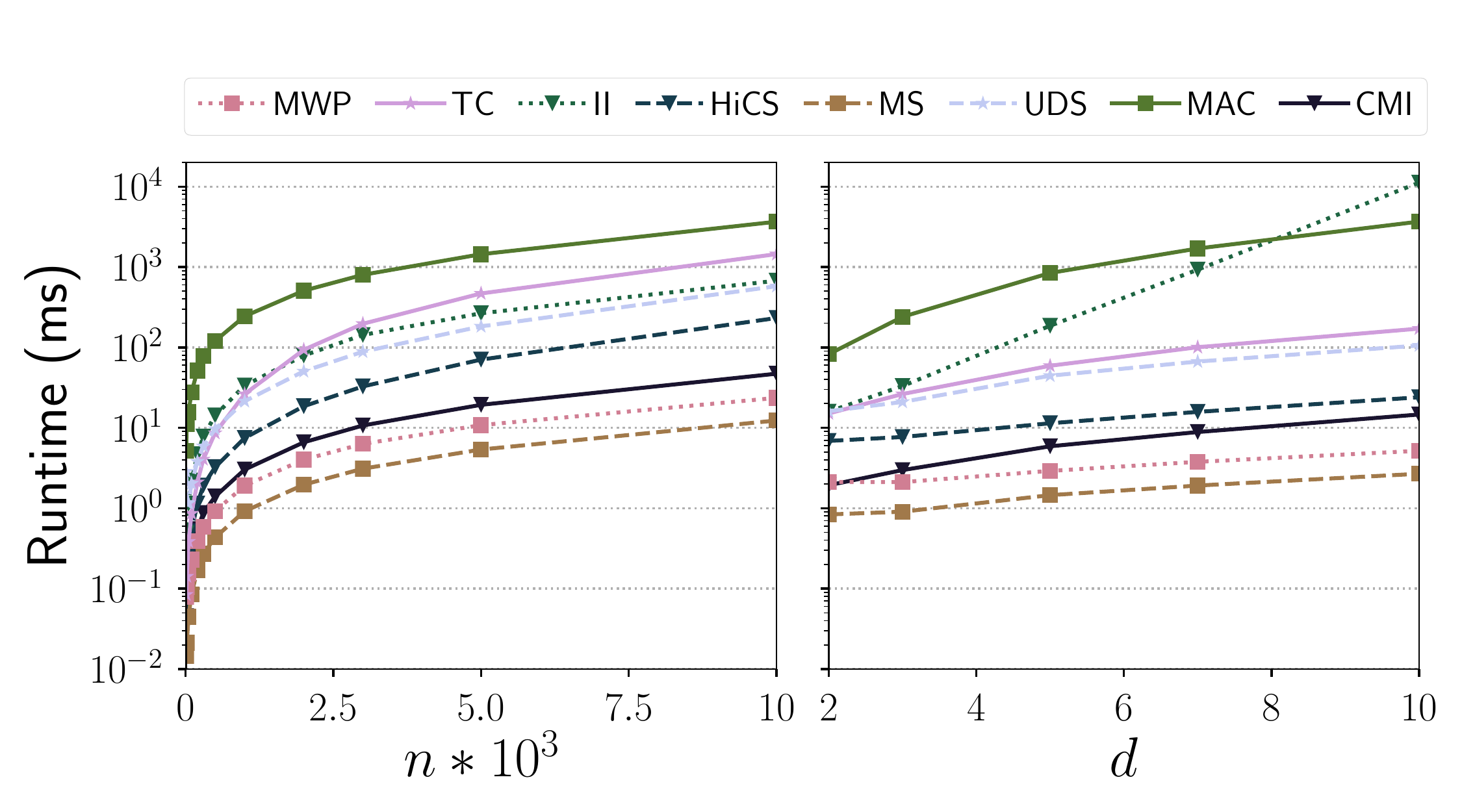} 
	\caption{Execution time w.r.t. $n$ and $d$}
	\label{fig:Runtime_n_d}
\end{figure}

We now look at the runtime requirements of our approach. We measure the average CPU time for each estimator against 500 data sets with a growing number of objects $n$ and dimensionality $d$. Note that which data set we use only has a marginal effect on the measured time. For consistency, we use instantiations of \textbf{I} for every estimator. 

Figure \ref{fig:Runtime_n_d} graphs the results. As we can see, \textit{MWP} is the second fastest after \textit{MS}. 
\textit{HiCS} and \textit{CMI} scale relatively well with $n$ and $d$. There is a second group formed by \textit{TC}, \textit{II} and \textit{UDS} one order of magnitude slower. However, \textit{II} does not scale well with $d$. \textit{MAC} is way behind all others. 

One should note that the runtime of \textit{MWP} can be further improved via parallelisation and prior indexing.

\subsection{Discussion}

Our study has shown that \textit{MWP} fulfils all the requirements we have laid out. We have compared \textit{MWP} to a range of multi\-variate (\textbf{R1}) and non-parametric (\textbf{R5}) approaches. 
We have shown to which extent they are efficient (\textbf{R2}), general-purpose (\textbf{R3}), interpretable (\textbf{R6}), sensitive to effect size (\textbf{R7}) and robust (\textbf{R8}).
Each approach, except \textit{MWP} and \textit{MS}, has at least one unintuitive parameter (\textbf{R4}): \textit{TC} and \textit{II} require $k \in \mathbb{N}$, \textit{CMI} requires $Q \in \mathbb{N}$, \textit{MAC} requires $\epsilon \in (0,1)$, \textit{UDS} requires $\beta \in \mathbb{N}$, \textit{HiCS} requires $\alpha \in (0,1)$. 
Next, only \textit{MWP} and \textit{HiCS} allow to trade accuracy for a computational advantage (\textbf{R9}). 
Table~\ref{table:requirements} summarizes our findings.

\begin{table}[ht]
	
	\centering
	\begin{tabular}{l c c c c c c c c c} 
		\hline 
		Estimator & \textbf{R1} & \textbf{R2} & \textbf{R3} & \textbf{R4} & \textbf{R5} & \textbf{R6} & \textbf{R7} & \textbf{R8} & \textbf{R9}\\ 
		\hline 
		\textit{MS} & \cmark & ++ & \xmark & \cmark & \cmark & \cmark & \xmark & \cmark & \xmark \\ 
		\textit{TC} & \cmark & - & \cmark & \xmark & \cmark & \xmark & \cmark & \xmark & \xmark\\
		\textit{II} & \cmark & -{}- & \xmark & \xmark & \cmark & \xmark & \xmark & \xmark & \xmark\\
		\textit{CMI} & \cmark & + & \xmark & \xmark & \cmark & \xmark & \xmark & \xmark & \xmark\\
		\textit{MAC} & \cmark & -{}- & \xmark & \xmark & \cmark & \xmark & \xmark & \cmark & \xmark\\
		\textit{UDS} & \cmark & - & \xmark & \xmark & \cmark  & \xmark & \xmark & \cmark & \xmark\\ 
		\textit{HiCS} & \cmark & + & \cmark & \xmark & \cmark &  \xmark & \xmark & \xmark & \cmark \\ 
		\textit{MWP} & \cmark & ++
		& \cmark & \cmark & \cmark & \cmark & \cmark & \cmark & \cmark \\
		\hline
	\end{tabular}
	\caption{Requirement fulfilment} 
	\label{table:requirements}
\end{table}

All in all, \textit{MWP} establishes itself as a state-of-the-art estima\-tor: It is versatile, allowing quality-runtime trade-offs and parallelisation, which is useful when time is critical, e.g., in large data streams. At the same time, it shows very good detection quality with no restriction on the dependency type, while being easy to use and interpret. \textit{MWP} features a unique blend of desirable properties that so far no competitor offers. 

\section{Conclusions \& Outlook}
\label{conclusions}

In this paper, we have introduced \textit{MCDE}, a framework to estimate multivariate dependency, and its instantiation as \textit{MWP}. We have shown that \textit{MWP} fulfils all the requirements one would expect from a state-of-the-art dependency estimator. 
Compared to other approaches, it provides high statistical power on a large panel of dependencies, while being very efficient. Thus, \textit{MCDE}-\textit{MWP} is particularly promising for correlation monitoring in data streams. 

As future work, we will study the deployment of \textit{MCDE} in streaming scenarios. Our goal is to characterize the \textbf{anytime flexibility} of \textit{MCDE} by refining the bound presented in Theorem \ref{"th:hoeffding-chernoff-contrast"} via further assumptions.
It will also be interesting to consider different instantiations of the statistical test, e.g., by comparing recent modifications of the \textit{Mann-Whitney\,U} test, such as \cite{fligner1981robust} and \cite{brunner2000nonparametric}. 
Finally, the efficiency of \textit{MCDE} in the streaming setting could be further improved via efficient insert and delete index operations. 

\section*{Acknowledgements}
This work was partially supported by the DFG Research Training
Group 2153: `Energy Status Data -- Informatics Methods for its
Collection, Analysis and Exploitation' and the German
Federal Ministry of Education and Research (BMBF) via Software Campus (01IS17042).

\balance

\bibliographystyle{IEEEtran}
\bibliography{IEEEabrv,mcde18-bibliography}

\end{document}